\documentclass{article} 
\usepackage{iclr2027_conference,times}

\iclrfinalcopy 

\usepackage[table]{xcolor}

\usepackage[
    pagebackref=true,
    breaklinks=true,
    colorlinks=true,
    urlcolor={violet!90!black},
    linkcolor={violet!90!black},
    citecolor={blue!90!black},
    bookmarks=false
]{hyperref}

\usepackage{amsmath,amsfonts,bm}

\def\eqref#1{equation~\ref{#1}}

\def\1{\bm{1}}

\DeclareMathAlphabet{\mathsfit}{\encodingdefault}{\sfdefault}{m}{sl}
\SetMathAlphabet{\mathsfit}{bold}{\encodingdefault}{\sfdefault}{bx}{n}

\usepackage{url}

\usepackage{latexsym}
\usepackage[T1]{fontenc}
\usepackage[utf8]{inputenc}
\usepackage{microtype}
\usepackage{inconsolata}
\usepackage{orcidlink}
\usepackage{amsmath}
\usepackage{amssymb}
\usepackage{amsfonts}
\usepackage{mathtools}
\usepackage{amsthm}
\usepackage{nicefrac}

\usepackage{booktabs}
\usepackage{multirow}
\usepackage{threeparttable}
\usepackage{tabularx}

\usepackage{graphicx}
\usepackage{subcaption}
\usepackage{wrapfig}
\graphicspath{{./}{figures/}{assets/}{analysis_out/}}

\usepackage{tikz}
\usepackage{pgfplots}
\usetikzlibrary{intersections}
\usepgfplotslibrary{fillbetween,groupplots,dateplot}
\pgfplotsset{compat=newest, scaled y ticks=false}

\usepackage{algorithm}
\usepackage{algpseudocode}

\usepackage{bookmark}
\usepackage[capitalize,noabbrev]{cleveref}

\usepackage{pifont}

\usepackage{xspace}

\theoremstyle{plain}
\newtheorem{theorem}{Theorem}[section]
\newtheorem{proposition}[theorem]{Proposition}

\theoremstyle{definition}

\theoremstyle{remark}
\newtheorem{remark}[theorem]{Remark}

\usepackage{etoc}
\usepackage{newtxtext}

\newcommand{\method}{DiffGate\xspace}

\definecolor{headergray}{RGB}{242,244,247}
\definecolor{methodblue}{RGB}{235,244,252}
\definecolor{privblue}{RGB}{246,249,253}
\definecolor{teachergray}{RGB}{248,248,248}
\definecolor{sectiongray}{RGB}{248,249,251}
\definecolor{trainyellow}{RGB}{252,248,238}

\usepackage[most]{tcolorbox}
\usepackage{listings}
\usepackage{xcolor}

\definecolor{qualgray}{RGB}{247,248,250}
\definecolor{qualborder}{RGB}{205,211,218}

\definecolor{grpobg}{RGB}{253,247,247}
\definecolor{grpoborder}{RGB}{194,128,128}

\definecolor{methodbg}{RGB}{243,248,253}
\definecolor{methodborder}{RGB}{92,143,190}

\definecolor{takeawaybg}{RGB}{250,249,244}
\definecolor{takeawayborder}{RGB}{191,181,143}

\lstdefinestyle{qualcode}{
    basicstyle=\ttfamily\scriptsize,
    columns=fullflexible,
    breaklines=true,
    breakatwhitespace=false,
    keepspaces=true,
    showstringspaces=false,
    frame=none,
    tabsize=4,
    aboveskip=0pt,
    belowskip=0pt
}

\newtcolorbox{exampleheader}[1]{
    enhanced,
    breakable,
    colback=white,
    colframe=qualborder,
    boxrule=0.6pt,
    arc=2pt,
    left=2mm,
    right=2mm,
    top=1.8mm,
    bottom=1.8mm,
    before skip=7pt,
    after skip=3pt,
    title={#1},
    fonttitle=\bfseries\small,
    coltitle=black,
    colbacktitle=qualgray,
    attach boxed title to top left={
        xshift=2mm,
        yshift=-1.5mm
    },
    boxed title style={
        colback=qualgray,
        colframe=qualborder,
        boxrule=0.5pt,
        arc=2pt
    }
}

\newtcolorbox{problembox}{
    enhanced,
    breakable,
    colback=qualgray,
    colframe=qualborder,
    boxrule=0.5pt,
    arc=2pt,
    left=2mm,
    right=2mm,
    top=1.2mm,
    bottom=1.2mm,
    before skip=2pt,
    after skip=3pt,
    title=\textbf{Problem},
    fonttitle=\small,
    coltitle=black
}

\newtcolorbox{grpotext}{
    enhanced,
    breakable,
    colback=grpobg,
    colframe=grpoborder,
    boxrule=0.55pt,
    arc=2pt,
    left=2mm,
    right=2mm,
    top=1.2mm,
    bottom=1.2mm,
    before skip=2pt,
    after skip=3pt,
    title=\textbf{GRPO},
    fonttitle=\small,
    coltitle=black
}

\newtcolorbox{methodtext}{
    enhanced,
    breakable,
    colback=methodbg,
    colframe=methodborder,
    boxrule=0.65pt,
    arc=2pt,
    left=2mm,
    right=2mm,
    top=1.2mm,
    bottom=1.2mm,
    before skip=2pt,
    after skip=3pt,
    title=\textbf{\method},
    fonttitle=\small,
    coltitle=black
}

\newtcblisting{grpocode}{
    enhanced,
    breakable,
    listing only,
    colback=grpobg,
    colframe=grpoborder,
    boxrule=0.55pt,
    arc=2pt,
    left=2mm,
    right=2mm,
    top=1mm,
    bottom=1mm,
    before skip=2pt,
    after skip=3pt,
    title=\textbf{GRPO},
    fonttitle=\small,
    coltitle=black,
    listing options={style=qualcode}
}

\newtcblisting{methodcode}{
    enhanced,
    breakable,
    listing only,
    colback=methodbg,
    colframe=methodborder,
    boxrule=0.65pt,
    arc=2pt,
    left=2mm,
    right=2mm,
    top=1mm,
    bottom=1mm,
    before skip=2pt,
    after skip=3pt,
    title=\textbf{\method},
    fonttitle=\small,
    coltitle=black,
    listing options={style=qualcode}
}

\title{DiffGate: Difficulty-Gated Teacher Guidance for On-Policy Distillation}

\author{
Karn Tiwari\textsuperscript{1,*},
Varnith Chordia\textsuperscript{2,*},
Prathosh A.P.\textsuperscript{1} \\
\textsuperscript{1}Indian Institute of Science, Bengaluru, India \\
\textsuperscript{2}Snap Inc. \\
\texttt{\{karntiwari,prathosh\}@iisc.ac.in} \\
\texttt{vchordia@snapchat.com}\\[0.5em]
{\small \textsuperscript{*}Equal contribution.}
}

\begin{document}

\maketitle

\begin{abstract}
On-policy distillation (OPD) has emerged as a widely used paradigm for post-training large language models, reducing the train--test mismatch of conventional distillation by supervising the student on its own generated trajectories. However, existing OPD objectives remain largely token-local and outcome-agnostic, optimizing teacher--student agreement at each prefix despite reasoning quality being determined at the trajectory level. Reinforcement learning with verifiable rewards (RLVR), particularly Group Relative Policy Optimization (GRPO), provides complementary outcome-level supervision but suffers from sparse rewards and coarse credit assignment. We show that OPD and RLVR exhibit complementary blind spots: teacher signals provide dense local guidance but are weakly aligned with rollout correctness, whereas group-relative rewards capture task success but provide coarse token-level credit and vanish on all-failure groups. We introduce \method, an outcome-gated objective that combines GRPO with selective, bounded teacher guidance. Teacher supervision is applied only to failed trajectories, scaled by group difficulty, and smoothly bounded to prevent extreme teacher--student discrepancies from dominating optimization. The verifier therefore determines \emph{which trajectories} receive teacher guidance, while the teacher provides dense token-level update directions within those trajectories. Across Qwen3-0.6B and Qwen3-1.7B students, \method improves code avg@8 over matched GRPO by $+1.7$ and $+1.8$ points and pass@8 by $+1.6$ and $+5.7$ points, respectively. On mathematics, avg@8 remains within $0.5$ points of GRPO while pass@8 improves by $+1.1$ and $+3.9$ points. Overall, \method improves pass@8 across all four model--domain settings, demonstrating improved solution coverage under our evaluation protocol.
\end{abstract}

\etocdepthtag.toc{mtmain}

\section{Introduction}
Knowledge distillation (KD) transfers capabilities from a teacher to a smaller student by matching their predictive distributions. On-policy distillation (OPD) reduces the train--test mismatch of conventional KD by training on trajectories sampled from the student and querying the teacher on the states the student actually visits. A common OPD formulation minimizes the reverse KL divergence between the student and teacher on these visited states, yielding dense token-level supervision from their probability differences. However, this supervision remains \emph{local and outcome-agnostic}: tokens are updated based on distributional disagreement regardless of whether the complete trajectory is correct. Consequently, teacher disagreement may reflect stylistic differences rather than reasoning errors, while large probability gaps can produce disproportionately large updates. Thus, OPD provides dense on-policy supervision, but not necessarily supervision aligned with task success \citep{hinton2015distilling, lu2025onpolicydistillation, li2026rethinking, gan2026teacher, fu2026revisiting, yang2026learning}.

Reinforcement learning with verifiable rewards (RLVR) provides a complementary source of supervision. Methods such as GRPO directly optimize sequence-level correctness using verifiable outcomes, including mathematical answers and executable code. However, this supervision is inherently coarse: a single trajectory-level advantage is propagated to all tokens, limiting fine-grained credit assignment. More critically, the group-relative advantage vanishes when all sampled trajectories receive the same reward, leaving difficult all-failure groups without a learning signal. RLVR can further suffer from sparse rewards and policy entropy collapse during optimization. OPD and GRPO therefore exhibit complementary blind spots: OPD is \emph{dense but correctness-blind}, whereas GRPO is \emph{correctness-aware but coarse and sparse} \citep{nie2026gradient, liu2026dapo, mishra2026policy, mroueh2025reinforcement, shao2024deepseekmath, wen2026reinforcement}.

This complementarity motivates integrating OPD with RLVR; however, naively combining their learning signals can introduce conflicting supervision and hinder optimization. As illustrated in Figure~\ref{fig:blindness}, teacher--student disagreement alone cannot reliably distinguish correct from incorrect trajectories or localize the decisive reasoning error. We therefore use outcome feedback to route teacher supervision at the trajectory level rather than treating teacher disagreement as a token-level credit-assignment signal. Uniform distillation may therefore modify already-correct solutions, while large teacher--student discrepancies can dominate updates even when outcome feedback is informative. Moreover, strong teacher imitation may restrict exploration and interfere with reward-driven learning. Our training diagnostics further show that teacher disagreement can be larger on correct trajectories than on failed ones, while reward-only training can substantially reduce policy entropy (see Appendix Section~\ref{app:adv_dynamics}). These observations motivate a simple principle: apply bounded teacher guidance selectively to failed trajectories, concentrate it where outcome supervision is weakest, and reduce its influence as the student improves \citep{wang2026distilled, cai2026h}.

We introduce \method, a unified post-training objective that injects teacher knowledge directly into the RL update through \emph{selective, bounded teacher guidance}. For each failed trajectory, \method converts the teacher--student log-probability gap into a bounded token-level signal and scales it by the difficulty of the sampled group. Teacher guidance is therefore strongest on all-failure groups, where GRPO provides no signal, decreases as the group solve rate improves, and vanishes on successful trajectories. Bounding the teacher term preserves its local update direction while preventing extreme discrepancies from dominating optimization. In this way, the verifier determines \emph{which failed trajectories} receive teacher guidance, while the teacher provides dense token-level update directions within those trajectories. Our main contributions are summarized as follows:

\begin{itemize}
    \item \textbf{Local--global supervision gap.}
    We identify complementary blind spots in OPD and GRPO: OPD provides dense but outcome-agnostic token supervision, whereas GRPO captures trajectory correctness but offers coarse credit and vanishes on reward-uniform groups.
    \item \textbf{Difficulty-gated teacher guidance.} We introduce \method, which integrates bounded teacher guidance directly into the GRPO update, activating it only on failed trajectories and scaling it by group difficulty. This concentrates local supervision where outcome feedback is weakest while preserving successful trajectories.
    \item \textbf{Comprehensive empirical evaluation.} We evaluate \method\ against GRPO and distillation baselines on mathematical reasoning and code generation across Qwen3 student sizes, assess out-of-distribution transfer to science benchmarks, and analyze the effects of difficulty gating, bounded teacher guidance, rollout group size, and teacher strength.
\end{itemize}

\section{Related Work}
\label{sec:related}
\textbf{On-Policy Distillation (OPD).} OPD trains the student on its own generated trajectories and queries a teacher on the visited states, reducing the distribution mismatch of conventional distillation \citep{agarwal2024policy, li2026rethinking, xu2026deepseek, team2026kimi}. Recent work has focused on improving the reliability of this dense token-level supervision. EOPD adapts the divergence according to teacher uncertainty to mitigate the mode-seeking behavior of reverse KL \citep{jin2026entropy}, while Uni-OPD identifies insufficient exploration and unreliable teacher supervision as key bottlenecks and calibrates token guidance using outcome information \citep{hou2026uni}. Other analyses show that OPD can fail when teacher and student reasoning patterns are poorly aligned, when the teacher provides little additional capability, or when teacher gradients are noisy on already-correct trajectories \citep{li2026rethinking, armandpour2026unmasking}. TGPO further studies unreliable negative supervision under large teacher--student policy divergence \citep{liu2026teacher}. These methods improve \emph{which} teacher signals are trusted, but standard OPD remains primarily token-local: teacher--student disagreement alone does not directly encode whether the full trajectory succeeds.

\textbf{Reinforcement Learning with Verifiable Rewards (RLVR).} RLVR directly optimizes task success using automatically checkable outcomes, with GRPO becoming a widely used critic-free approach for mathematical and code reasoning \citep{shao2024deepseekmath, guo2025deepseek}. However, outcome rewards provide sparse and coarse supervision: GRPO assigns the same trajectory-level advantage to every token and yields no gradient when all responses in a group receive the same reward \citep{mishra2026policy, nie2026gradient}. Prior work addresses related limitations through dynamic sampling \citep{yu2026dapo}, correcting optimization and length biases \citep{liu2025understanding, sui2025stop}, or introducing finer-grained credit assignment \citep{li2026outcome, jiang2025rethinking, wang2026targeted}. These approaches improve outcome supervision, whereas \method complements it with dense teacher guidance where the group-relative signal is weak.

\textbf{Combining OPD and RLVR.}
Recent work has explored combining teacher supervision with outcome-based policy optimization. KDRL jointly optimizes GRPO and reverse-KL distillation, but its teacher signal is not explicitly bounded at the token level \citep{xu2025kdrl}. RL-aware distillation uses selective teacher imitation and a trust-region objective to balance distillation and reward optimization \citep{zhang2026reinforcement}. CoDistill-GRPO jointly trains large and small policies using outcome-based and distillation rewards, while Uni-OPD calibrates teacher supervision using outcome information \citep{kwon2026codistill,hou2026uni}. In contrast, \method\ combines bounded teacher guidance on failed trajectories with group-difficulty weighting, providing supervision when GRPO has no signal and reducing it as the group solve rate increases.

\begin{figure}[t]
    \centering
    \includegraphics[width=\linewidth]{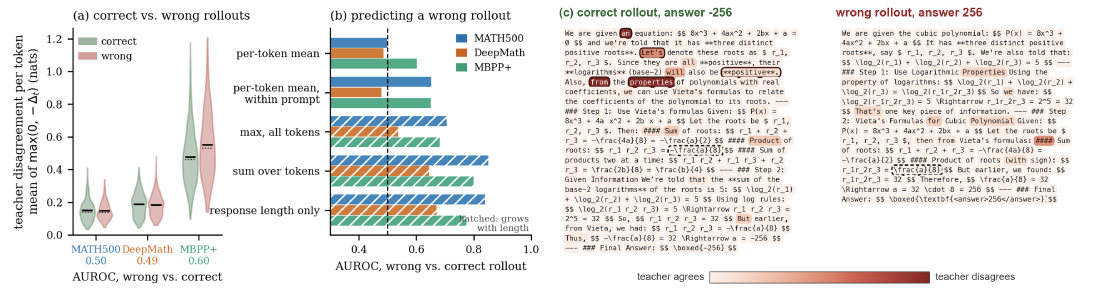}
    \caption{\small \textbf{Teacher disagreement alone does not reliably identify incorrect rollouts.} For a distillation-only Qwen3-1.7B student evaluated with a frozen Qwen3-8B teacher: (a) average token-level teacher--student disagreement distinguishes correct from incorrect rollouts only at chance level on mathematics (AUROC $0.50$ and $0.49$) and weakly on code (AUROC $0.60$); (b) aggregate disagreement becomes more predictive mainly when it also reflects response length; and (c) within an incorrect rollout, large disagreement can occur on ordinary tokens rather than on the token associated with the decisive error. These observations motivate \method\ to use the verifier to determine \emph{which} trajectories receive teacher guidance, while the teacher provides the local token-level update direction. Further analysis is provided in Appendix Section~\ref{sec:opd_failure}.}
    \label{fig:blindness}
    \vspace{-0.5cm}
\end{figure}

\section{Background and Preliminaries}
\label{sec:preliminaries}

\textbf{Problem Setup and Notations.}
Let $\pi_{\mathrm{T}}$ denote a fixed teacher policy and $\pi_\theta$ a student policy parameterized by $\theta$. Given a prompt $x\sim\mathcal{D}$, the student autoregressively generates a response as $y=(y_1,\ldots,y_L)\sim\pi_\theta(\cdot\mid x)$, where $y_t \sim \pi_\theta(\cdot\mid c_t), c_t=(x,y_{<t}),$ and $\pi_\theta(y\mid x)=\prod_{t=1}^{L}\pi_\theta(y_t\mid c_t)$. Because $y$ is sampled from the student, the visited contexts $c_t$ are induced by its own policy. We consider two sources of supervision on these trajectories: a teacher distribution $\pi_{\mathrm{T}}(\cdot\mid c_t)$ that provides token-level feedback, and a verifier reward $R(x,y)$ that evaluates the complete trajectory.

\textbf{On-Policy Distillation.} OPD supervises the student on prefixes generated by the student itself, reducing the mismatch between training and inference-time states. A standard formulation minimizes the reverse KL between the student and teacher distributions at student-visited contexts as follows:
\begin{align}
    \mathcal{L}_{\mathrm{OPD}}(\theta)
    =
    \mathbb{E}_{\substack{
        x\sim\mathcal{D},\\
        y\sim\pi_\theta(\cdot\mid x)
    }}
    \left[
        \frac{1}{L}
        \sum_{t=1}^{L}
        \log
        \frac{
            \pi_\theta(y_t\mid c_t)
        }{
            \pi_{\mathrm{T}}(y_t\mid c_t)
        }
    \right].
    \label{eq:opd_rkl}
\end{align}
Rather than evaluating the divergence over the full vocabulary, sampled-token OPD uses the student-generated token $y_t$ to construct a Monte Carlo approximation of the learning signal. For a fixed context $c$, maximizing the negative reverse KL yields a policy-gradient form as follows:
\begin{align}
    \nabla_\theta J_{\mathrm{OPD}}(\theta;c)
    =
    \mathbb{E}_{a\sim\pi_\theta(\cdot\mid c)}
    \left[
        A(a,c)\,
        \nabla_\theta
        \log\pi_\theta(a\mid c)
    \right],
\end{align}
where the token-level OPD advantage is $A(a,c) \coloneqq \log\pi_{\mathrm{T}}(a\mid c) - \log\pi_\theta(a\mid c).$
For a sampled token $y_t$, we therefore use $A_t = \operatorname{sg}\!\left[\log\pi_{\mathrm{T}}(y_t\mid c_t) -
\log\pi_\theta(y_t\mid c_t) \right],$
optimized via the Monte Carlo sampled gradient estimator. During training, the sampled-token advantage is computed from the behavior policy and treated as a stop-gradient quantity. Thus, $A_t>0$ increases the probability of tokens favored more by the teacher, while $A_t<0$ suppresses tokens over-weighted by the student. This provides a dense token-level learning signal at every generation step.

\textbf{RL with Verifiable Rewards.} RL with verifiable rewards (RLVR) optimizes sequence-level task success using an automatic verifier. For reasoning and code tasks, we consider a binary reward $R(x,y)\in\{0,1\}$
where $R(x,y)=1$ if the generated response $y$ is correct. The objective is:
\begin{align}
    J_{\mathrm{RLVR}}(\theta)
    =
    \mathbb{E}_{\substack{
        x\sim\mathcal{D},\\
        y\sim\pi_\theta(\cdot\mid x)
    }}
    \left[R(x,y)\right].
\end{align}
Using REINFORCE, the corresponding policy gradient is given by:
\begin{align}
    \nabla_\theta J_{\mathrm{RLVR}}(\theta)
    =
    \mathbb{E}
    \left[
        R(x,y)
        \sum_{t=1}^{L}
        \nabla_\theta
        \log\pi_\theta(y_t\mid c_t)
    \right].
    \label{eq:rlvr_gradient}
\end{align}
Unlike the token-level OPD signal, the verifier reward evaluates the complete response and does not identify which individual tokens contributed to its success or failure. Thus, unlike OPD, RLVR assigns the same trajectory-level signal to all tokens in a response. GRPO constructs a prompt-relative learning signal by comparing rewards within each rollout group, by sampling a group of $n$ responses $y^{(i)}\sim\pi_\theta(\cdot\mid x)$ for the same prompt, with rewards $R_i=R(x,y^{(i)})$. Let $\bar R = \frac{1}{n}\sum_{j=1}^{n}R_j, \sigma_R = \sqrt{\frac{1}{n-1}\sum_{j=1}^{n}(R_j-\bar R)^2}$. The trajectory-level GRPO advantage is $A_i^{\mathrm{GRPO}} = \frac{R_i-\bar R}{\sigma_R+\varepsilon}.\label{eq:grpo_advantage}$Since $A_i^{\mathrm{GRPO}}$ is defined for the complete rollout, it is shared across all tokens in trajectory $i$. The corresponding on-policy surrogate is:
\begin{align}
    \mathcal{L}_{\mathrm{GRPO}}
    =
    -
    \frac{1}{n}
    \sum_{i=1}^{n}
    \frac{1}{L_i}
    \sum_{t=1}^{L_i}
    \operatorname{sg}
    \left[
        A_i^{\mathrm{GRPO}}
    \right]
    \log
    \pi_\theta
    \left(
        y_t^{(i)}\mid c_{i,t}
    \right).
    \label{eq:grpo_objective}
\end{align}
In practice, we optimize its PPO-style clipped surrogate using importance ratios with respect to the frozen behavior policy. When all trajectories receive the same reward, their GRPO advantages vanish.

\textbf{\method\ Motivation.}
OPD and GRPO provide complementary supervision. OPD supplies dense token-level teacher guidance but does not directly encode whether the complete trajectory is correct. GRPO captures trajectory-level correctness, but assigns the same advantage to every token and provides no learning signal on reward-uniform groups. Moreover, Figure~\ref{fig:blindness} shows that teacher--student disagreement alone does not reliably distinguish correct from incorrect trajectories or localize the decisive reasoning error. These observations motivate \method: the verifier determines \emph{which} trajectories receive additional guidance, while the teacher provides dense token-level update directions within the selected failed trajectories. Importantly, \method\ performs trajectory-level routing and does not attempt to identify the causal error token within a response.

\section{Methodology}
\label{sec:method}
As discussed in Section ~\ref{sec:preliminaries}, OPD and GRPO provide complementary supervision: OPD gives dense token-level guidance but is local and unbounded, while GRPO captures trajectory correctness but is coarse and vanishes on reward-uniform groups. We introduce \method, which combines both, using teacher guidance when outcome supervision is weak. Figure~\ref{fig:main_figure} summarizes the method.

\begin{figure}[t]
    \centering
    \includegraphics[width=\linewidth]{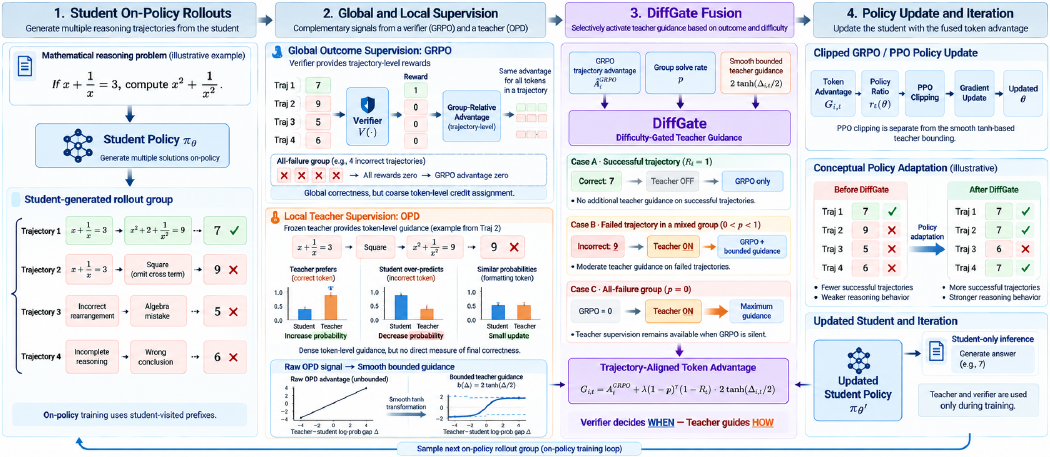}
    \caption{\small
    \textbf{Overview of \method.}
    The student generates on-policy trajectories that receive sequence-level correctness feedback from a verifier and token-level guidance from a frozen teacher. \method\ bounds the sampled-token teacher--student log-probability gap, applies teacher guidance only to failed trajectories, and scales its strength according to the group solve rate. The resulting token-level advantage combines local teacher guidance with the global GRPO outcome signal to update the student policy. See Algorithm~\ref{alg:method} for details.}
    \label{fig:main_figure}
    \vspace{-0.5cm}
\end{figure}

\textbf{Local versus Global Supervision.} OPD and RLVR admit a similar policy-gradient form,
\begin{align}
    \widehat{\nabla_\theta J}
    =
    \sum_t
    A_t\,
    \nabla_\theta
    \log\pi_\theta(y_t\mid c_t),
\end{align}
but differ in how the advantage is constructed. OPD provides the token-level advantage $A_{i,t}^{\mathrm{OPD}} = \log\pi_{\mathrm{T}}(y_t^{(i)}\mid c_{i,t}) - \log\pi_\theta(y_t^{(i)}\mid c_{i,t}),$ whereas GRPO assigns the same trajectory-level advantage $A_i^{\mathrm{GRPO}}$ to every token in rollout $i$. Thus, OPD provides \emph{dense but local} supervision, while GRPO provides \emph{global but coarse} feedback. However, the OPD advantage is an unbounded log-probability ratio: large teacher--student discrepancies can produce large token-level signals, potentially dominating the GRPO advantage. Moreover, the signals may disagree: teacher-preferred tokens can occur in incorrect trajectories, while correct trajectories may contain teacher-disfavored tokens.

\textbf{Bounded Local Teacher Guidance.} For each prompt $x$, we sample a group of $n$ trajectories
$\{y^{(1)},\ldots,y^{(n)}\}$ from the behavior policy $\pi_{\theta_{\mathrm{old}}}$. For token $y_t^{(i)}$ in trajectory $i$, with context
$c_{i,t}=(x,y_{<t}^{(i)})$, we define the sampled-token OPD gap as follows:
\begin{align}
    \Delta_{i,t}
    =
    \log \pi_{\mathrm{T}}
    (y_t^{(i)}\mid c_{i,t})
    -
    \log \pi_{\theta_{\mathrm{old}}}
    (y_t^{(i)}\mid c_{i,t}).
    \label{eq:local_gap}
\end{align}
Although $\Delta_{i,t}$ provides dense teacher supervision, it is unbounded and can attain large magnitude when the teacher and student strongly disagree. To prevent a small number of tokens from dominating the update, we apply a bounded transformation that keeps the OPD signal on a scale comparable to the GRPO advantage and defined as follows:
\begin{align}
    B_{i,t}
    =
    \phi(\Delta_{i,t})
    \coloneqq
    2\tanh\left(\frac{\Delta_{i,t}}{2}\right),
    \qquad
    B_{i,t}\in(-2,2).
    \label{eq:bounded_advantage}
\end{align}
The transformation preserves the sign of the original OPD advantage and is approximately linear around zero,
$\phi(\Delta)\approx\Delta$ for small $|\Delta|$, while smoothly saturating for large teacher--student discrepancies.

\textbf{Difficulty-Gated Trajectory Guidance.} Let $R_i\in\{0,1\}$ denote the verifier reward of trajectory $i$ and let $p=\frac{1}{n}\sum_{j=1}^{n}R_j \label{eq:solve_rate}$ denote the group solve rate. We use the standard group-relative advantage $A_i^{\mathrm{GRPO}}$ defined in Section ~\ref{sec:preliminaries} as the global outcome signal.

A central design choice is that teacher guidance should not be applied uniformly. In particular, we want the teacher to guide trajectories that fail the verifier, with stronger guidance on prompts that the student finds difficult. We therefore define the trajectory gate as follows:
\begin{align}
    w_i
    =
    (1-p)^\gamma(1-R_i),
    \label{eq:trajectory_gate}
\end{align}
where $\gamma\geq0$ controls the dependence on group difficulty. The factor $(1-R_i)$ removes teacher guidance from successful trajectories, while $(1-p)^\gamma$ increases its strength as the group's solving rate decreases. Combining the global and local signals gives the proposed token-level advantage as:
\begin{align}
    G_{i,t}
    =
    \underbrace{
        A_i^{\mathrm{GRPO}}
    }_{\text{global outcome}}
    +
    \underbrace{
        \lambda(1-p)^\gamma(1-R_i)
    }_{\text{trajectory gate}}
    \underbrace{
        2\tanh\left(\frac{\Delta_{i,t}}{2}\right)
    }_{\text{bounded local guidance}},
    \label{eq:diffgate_advantage}
\end{align}
where $\lambda\geq0$ controls the overall strength of teacher supervision.

The components serve complementary roles: \textsl{(1) Trajectory correctness:} $(1-R_i)$ activates teacher guidance only for failed trajectories. \textsl{(2) Group difficulty:} $(1-p)^\gamma$ scales teacher guidance more strongly when fewer rollouts succeed. \textsl{(3) Local direction:} $2\tanh(\Delta_{i,t}/2)$ provides a bounded token-level teacher signal. \textsl{(4) Global supervision:} $A_i^{\mathrm{GRPO}}$ provides trajectory-level outcome feedback whenever the rollout group contains reward variation.

Together, the verifier determines \emph{whether} correction is needed, the group solve rate determines \emph{how much} teacher guidance to apply, and the teacher determines \emph{how} to update the student locally. The gate operates at the trajectory level rather than the token level, since teacher--student disagreement alone does not reliably localize the exact reasoning error within a failed response.

The objective has useful boundary behavior. For a successful trajectory, $R_i=1$, and therefore 
\begin{align}
    G_{i,t}=A_i^{\mathrm{GRPO}},
\end{align}
so \method does not impose additional teacher pressure on a solution that already satisfies the verifier. Conversely, if all trajectories fail, then $p=0$ and $A_i^{\mathrm{GRPO}}=0$, giving 
\begin{align}
    G_{i,t} = 2\lambda \tanh\left(\frac{\Delta_{i,t}}{2}\right).
\label{eq:all_fail}
\end{align}
Thus, the teacher gate is maximal on all-failure groups, where GRPO provides no learning signal.

\textbf{Policy Optimization.}
We treat $G_{i,t}$ as a stop-gradient token-level advantage. For trajectories sampled from $\pi_{\theta_{\mathrm{old}}}$, we define the importance ratio $\rho_{i,t}(\theta)
=\frac{\pi_\theta(y_t^{(i)}\mid c_{i,t})}{
\pi_{\theta_{\mathrm{old}}}(y_t^{(i)}\mid c_{i,t})}.$ The student is optimized using the clipped policy surrogate as follows:
\begin{align}
    \mathcal{L}_{\method}(\theta)
    =
    -\frac{1}{n}
    \sum_{i=1}^{n}
    \frac{1}{L_i}
    \sum_{t=1}^{L_i}
    \min\Big(
        \rho_{i,t}(\theta)\,
        \operatorname{sg}[G_{i,t}],
        \operatorname{clip}
        \big(\rho_{i,t}(\theta),1-\epsilon_{\mathrm{clip}},1+\epsilon_{\mathrm{clip}}\big)\,
        \operatorname{sg}[G_{i,t}]
    \Big).
    \label{eq:diffgate_loss}
\end{align}
Here, $\operatorname{sg}[\cdot]$ denotes the stop-gradient operator, so gradients are taken only through the current policy $\pi_\theta$. Setting $\lambda=0$ recovers GRPO, while setting $\gamma=0$ removes group-difficulty weighting but retains teacher guidance on failed trajectories, yielding the additive GRPO$+$OPD baseline. These controls isolate the effects of teacher guidance and difficulty gating, respectively.

\textbf{Theoretical Analysis.} We establish three properties of \method. First, the bounded teacher signal admits a fixed-prefix policy-gradient interpretation through an $f$-divergence while limiting large teacher--student discrepancies. Second, difficulty gating provides guidance on all-failure groups, where GRPO is silent, and reduces its weight as the solve rate increases. Third, both the teacher contribution and combined token-level advantage are bounded. These results characterize how \method\ balances local teacher guidance and outcome supervision, without implying convergence of the full PPO objective. Proofs are provided in Appendix Section~\ref{app:theory}.

\section{Experiments}
\label{sec:experiments}
We evaluate \method on mathematics and code against both distillation-only and outcome-only baselines, focusing on four central questions: \textbf{(1)} Does combining outcome-level RLVR with token-level teacher guidance improve over OPD and GRPO alone? \textbf{(2)} Does the method improve out-of-domain transfer to unseen tasks? \textbf{(3)} When does teacher guidance contribute during training, and can generic regularization explain its gains? \textbf{(4)} Which design choices matter most, including difficulty gating, bounded teacher guidance, rollout group size, and teacher strength?

\textbf{Models and Training Data.} We use Qwen3-0.6B and Qwen3-1.7B students with a fixed Qwen3-8B teacher \citep{yang2025qwen3}. We train on \textsc{DeepMath} for mathematics \citep{he2026deepmath} and \textsc{TACO} for code \citep{li2023taco}, using answer verification and execution-based tests, respectively.

\textbf{Evaluation Benchmarks.} For mathematics, we evaluate on five challenging benchmarks: \textsc{Minerva-Math}~\citep{lewkowycz2022solving}, \textsc{OlympiadBench}~\citep{he2024olympiadbench}, \textsc{AIME24}~\citep{aime24}, \textsc{AIME25}~\citep{aime25}, and \textsc{HMMT25}~\citep{dekoninck2026matharena}. We exclude easier benchmarks such as \textsc{GSM8K}~\citep{cobbe2021training}, \textsc{MATH500}~\citep{lightman2024lets}, and \textsc{AMC23}~\citep{mathai2023amc23} from the main evaluation. For code generation, we use \textsc{MBPP+} and \textsc{HumanEval+}~\citep{liu2023evalplus}, together with \textsc{LiveCodeBench}~\citep{jain2025livecodebench}. To evaluate out-of-distribution transfer, we directly test the mathematics-trained models on science benchmarks without any additional science training, using \textsc{GPQA}~\citep{rein2023gpqa}, \textsc{MMLU-STEM}~\citep{hendrycks2020measuring}, and \textsc{OlympiadBench-Physics}~\citep{he2024olympiadbench}.

\begin{table}[t]
    \centering
    \small
    \renewcommand{\arraystretch}{1.12}
    \setlength{\tabcolsep}{3.2pt}
    \resizebox{\textwidth}{!}{%
    \begin{tabular}{llcccccccccc|cc|cc}
        \toprule
        \rowcolor{headergray}
        \multirow{2}{*}{\textbf{Student}}
        & \multirow{2}{*}{\textbf{Method}}
        & \multicolumn{2}{c}{\textbf{Minerva}}
        & \multicolumn{2}{c}{\textbf{Olympiad}}
        & \multicolumn{2}{c}{\textbf{AIME24}}
        & \multicolumn{2}{c}{\textbf{AIME25}}
        & \multicolumn{2}{c}{\textbf{HMMT25}}
        & \multicolumn{2}{c}{\textbf{Average}}
        & \multicolumn{2}{c}{\textbf{Science (OOD)}} \\
        
        \cmidrule(lr){3-4}\cmidrule(lr){5-6}\cmidrule(lr){7-8}
        \cmidrule(lr){9-10}\cmidrule(lr){11-12}\cmidrule(lr){13-14}\cmidrule(lr){15-16}
        \rowcolor{headergray}
        & & \textit{avg@8} & \textit{pass@8}
        & \textit{avg@8} & \textit{pass@8}
        & \textit{avg@8} & \textit{pass@8}
        & \textit{avg@8} & \textit{pass@8}
        & \textit{avg@8} & \textit{pass@8}
        & \textit{avg@8} & \textit{pass@8}
        & \textit{avg@8} & \textit{pass@8} \\
        \midrule

        \multirow{8}{*}{\textbf{0.6B}} & Base
        & 10.7 & 27.7 & 15.1 & 37.9
        & 2.5 & 6.6 & 0.8 & 6.8
        & 0.4 & 3.3 & 5.9 & 16.4
        & 21.7 & 48.8 \\

         & OPD (forward KL) 
        & 14.1 & 35.7 & 23.5 & 49.8
        & 3.8 & 15.6 & 3.8 & 16.4
        & 0.8 & 7.6 & 9.2 & 25.0
        & 29.4 & \textbf{59.1} \\

         & OPD (reverse KL)
        & 18.2 & 36.7 & 29.7 & \underline{54.7}
        & 2.5 & 15.4 & 7.9 & \underline{19.9}
        & 1.7 & 9.1 & 12.0 & 27.2
        & 30.1 & 58.2 \\

         & EOPD
        & 17.1 & 35.2 & 29.9 & 53.9
        & \underline{6.2} & 17.8 & 5.4 & 15.1
        & \underline{2.1} & 6.3 & 12.2 & 25.7
        & 30.2 & 58.1 \\

         & Uni-OPD 
        & 18.9 & \underline{36.9} & \underline{30.5} & \textbf{55.8}
        & \underline{6.2} & 15.9 & \underline{8.8} & \textbf{24.0}
        & 0.8 & 5.0 & 13.0 & 27.5
        & \underline{30.6} & 57.7 \\

         & SFT on teacher traces
        & 10.9 & 31.5 & 19.6 & 45.6
        & 5.0 & 12.2 & 2.5 & 11.4
        & 0.4 & 3.3 & 7.7 & 20.8
        & -- & -- \\

         & GRPO 
        & \textbf{20.0} & 34.9 & \textbf{31.6} & 51.7
        & \textbf{6.7} & \textbf{24.0} & \textbf{9.2} & \underline{19.9}
        & \textbf{3.8} & \underline{11.6} & \textbf{14.2} & \underline{28.4}
        & 22.8 & 50.1 \\

         & \textbf{\method}
        & \underline{19.4} & \textbf{38.4} & 30.3 & 54.6
        & \underline{6.2} & \underline{21.1} & \underline{8.8} & 19.2
        & \textbf{3.8} & \textbf{14.1} & \underline{13.7} & \textbf{29.5}
        & \textbf{31.2} & \underline{58.3} \\

        \midrule

        \multirow{8}{*}{\textbf{1.7B}} & Base
        & 28.6 & 46.6 & 38.1 & 63.2
        & \underline{14.6} & 33.7 & 10.0 & 21.5
        & 4.2 & \underline{14.2} & 19.1 & 35.8
        & 27.5 & 56.5 \\

         & OPD (forward KL) 
        & 25.0 & 47.4 & 31.2 & 59.6
        & 10.8 & 32.5 & 6.7 & 19.0
        & 2.1 & 9.1 & 15.2 & 33.5
        & 35.8 & \textbf{63.0} \\

         & OPD (reverse KL)
        & \underline{29.9} & \underline{49.6} & 39.0 & 63.1
        & 12.5 & 32.2 & \underline{11.7} & 26.1
        & 3.3 & 10.9 & 19.3 & 36.4
        & \underline{36.7} & 62.1 \\

         & EOPD 
        & 28.9 & 46.7 & 38.2 & 62.6
        & 12.5 & 36.3 & 10.4 & 23.8
        & 2.5 & 13.0 & 18.5 & 36.5
        & 36.2 & \textbf{63.0} \\

         & Uni-OPD 
        & 28.7 & 47.4 & 38.1 & \underline{63.9}
        & 13.8 & \underline{38.0} & 10.0 & 27.1
        & \underline{4.6} & \textbf{15.2} & 19.0 & \underline{38.3}
        & 36.5 & \textbf{63.0} \\

         & SFT on teacher traces
        & 26.0 & 48.0 & 33.1 & 59.8
        & 6.2 & 20.8 & 7.9 & 23.8
        & 1.7 & 8.2 & 15.0 & 32.1
        & -- & -- \\

         & GRPO
        & \textbf{31.2} & 46.9 & \textbf{41.8} & 61.1
        & \textbf{15.4} & 28.9 & \textbf{14.6} & \underline{27.2}
        & 4.2 & 11.9 & \textbf{21.4} & 35.2
        & 14.3 & 35.4 \\

         & \textbf{\method}
        & \textbf{31.2} & \textbf{52.9} & \underline{41.1} & \textbf{64.2}
        & \textbf{15.4} & \textbf{39.6} & \underline{11.7} & \textbf{28.1}
        & \textbf{5.8} & 10.8 & \underline{21.1} & \textbf{39.1}
        & \textbf{39.2} & \underline{62.4} \\

        \midrule

        \rowcolor{teachergray}
         & \textit{Teacher 8B}
        & 38.3 & 52.7 & 45.1 & 71.9
        & 23.3 & 45.7 & 17.1 & 36.9
        & 7.1 & 17.5 & 26.2 & 44.9
        & 47.6 & 69.7 \\

        \bottomrule
    \end{tabular}}
    \caption{\small \textbf{Mathematical reasoning results.} We report avg@8 and pass@8, with GRPO serving as the outcome-only baseline. \textbf{Bold} and \underline{underlined} values indicate the best and second-best student results, respectively. The science (OOD) columns report average performance on GPQA, MMLU-STEM, and OlympiadBench-Physics (Table~\ref{tab:science}), evaluated using the mathematics-trained students without additional science training.}
    \label{tab:math}
\end{table}

\textbf{Metrics.} We report \emph{avg@8} and \emph{pass@8}. Avg@8 measures mean correctness across sampled responses, while pass@8 measures whether at least one of eight samples is correct, capturing solution coverage. Both metrics are estimated from the same pool of $16$ generations per problem. Unless otherwise stated, we use temperature $0.6$, top-$p$ $1.0$ (no nucleus truncation), repetition penalty $1.15$, and a maximum generation length of $16{,}384$ tokens (More details are in Appendix Section \ref{appendix:implementation_details}).

\textbf{Baselines.} We compare \method\ with forward- and reverse-KL on-policy distillation (OPD)~\citep{lu2025onpolicydistillation}, supervised fine-tuning (SFT) on teacher-generated solutions, recent OPD frameworks such as EOPD~\citep{jin2026entropy}, Uni-OPD~\citep{hou2026uni}, and GRPO~\citep{shao2024deepseekmath}. GRPO serves as the matched outcome-only baseline and is recovered from \method\ by setting $\lambda=0$. Both methods use the same training data, rollout configuration, optimizer, and training-step budget.

\subsection{Main Results}
\label{sec:main_results}
Tables~\ref{tab:math} and~\ref{tab:code} report the main results on mathematical reasoning, code generation, and out-of-distribution (OOD) transfer to science benchmarks. We summarize three key findings below.

\textbf{Mathematical reasoning and code generation.} Across both 0.6B and 1.7B students, \method\ outperforms all OPD baselines in both avg@8 and pass@8, while also improving pass@8 over the GRPO baseline in all four settings. On mathematics, pass@8 improves over GRPO by $1.1$ and $3.9$ points, with only small changes in avg@8 ($-0.5$ and $-0.3$ points). On code, \method\ improves GRPO on both metrics: avg@8 increases by $1.7$ and $1.8$ points, while pass@8 increases by $1.6$ and $5.7$ points. Overall, \method\ combines the strengths of distillation and outcome-based learning, consistently improving solution coverage while outperforming OPD baselines across both domains.

\textbf{Out-of-distribution.} We evaluate the mathematics-trained students on GPQA, MMLU-STEM, and OlympiadBench-Physics without additional science training. \method\ achieves average science avg@8 scores of $31.2$ and $39.2$ for the 0.6B and 1.7B students, improving over the corresponding base models ($21.7$ and $27.5$) and the best OPD baselines ($30.6$ and $36.7$). The gains are driven primarily by GPQA and MMLU-STEM, while OlympiadBench-Physics changes relatively little. These results suggest that the benefits of difficulty-gated teacher guidance extend beyond the mathematics training domain, although the magnitude of transfer varies across science tasks (Table~\ref{tab:science}; Appendix~\ref{app:science}).

\textbf{Scaling to a 4B student.}
We further evaluate \method\ with a Qwen3-4B student, using a Qwen3-14B teacher for mathematics and a Qwen3-8B teacher for code. On mathematics, \method\ achieves $30.0$ avg@8 and $47.2$ pass@8, outperforming the strongest OPD baseline by $2.9$ avg@8 and $2.2$ pass@8 points (Table~\ref{tab:math4b}). On code, \method\ reaches $57.3$ avg@8 and $70.5$ pass@8, improving over the strongest OPD baseline by $2.5$ and $2.8$ points, respectively (Table~\ref{tab:code4b}). These results suggest that the benefits of combining outcome supervision with difficulty-gated teacher guidance extend to the larger 4B student setting, while the magnitude of the gains varies across tasks (Appendix Section~\ref{app:scale4b}).

\begin{table}[t]
    \centering
    \small
    \renewcommand{\arraystretch}{1.12}
    \setlength{\tabcolsep}{3.5pt}
    \resizebox{\textwidth}{!}{%
    \begin{tabular}{llcccccc|cc}
        \toprule
        \rowcolor{headergray}
        \multirow{2}{*}{\textbf{Student}}
        & \multirow{2}{*}{\textbf{Method}}
        & \multicolumn{2}{c}{\textbf{MBPP+}}
        & \multicolumn{2}{c}{\textbf{HumanEval+}}
        & \multicolumn{2}{c}{\textbf{LCB}}
        & \multicolumn{2}{c}{\textbf{Average}} \\

        \cmidrule(lr){3-4}\cmidrule(lr){5-6}
        \cmidrule(lr){7-8}\cmidrule(lr){9-10}
        \rowcolor{headergray}
        & & \textit{avg@8} & \textit{pass@8}
        & \textit{avg@8} & \textit{pass@8}
        & \textit{avg@8} & \textit{pass@8}
        & \textit{avg@8} & \textit{pass@8} \\
        \midrule

        \multirow{7}{*}{\textbf{0.6B}} & Base
        & 30.1 & 53.4 & 22.6 & 45.5
        & 6.1 & 14.1 & 19.6 & 37.7 \\

         & OPD (forward KL) 
        & 37.8 & 63.7 & 31.4 & 60.8
        & 7.7 & 18.7 & 25.6 & 47.7 \\

         & OPD (reverse KL)
        & \underline{46.0} & \textbf{65.8} & 33.4 & 56.5
        & 10.4 & 20.8 & 29.9 & 47.7 \\

         & EOPD
        & \underline{46.0} & 65.3 & 38.5 & 58.8
        & 10.3 & 21.5 & \underline{31.6} & \underline{48.6} \\

         & Uni-OPD 
        & 45.2 & \underline{65.5} & 36.9 & 58.1
        & 10.8 & 20.5 & 31.0 & 48.0 \\

         & GRPO
        & 40.9 & 58.5 & \underline{40.1} & \textbf{62.2}
        & \textbf{13.1} & \textbf{23.8} & 31.3 & 48.2 \\

        \rowcolor{methodblue}
         & \textbf{\method}
        & \textbf{46.1} & \underline{65.5} & \textbf{41.9} & \underline{61.8}
        & \underline{11.0} & \underline{22.0} & \textbf{33.0} & \textbf{49.8} \\

        \midrule

        \multirow{8}{*}{\textbf{1.7B}} & Base
        & 51.6 & 61.5 & 48.2 & 61.6
        & 14.7 & 25.5 & 38.2 & 49.5 \\

         & OPD (forward KL) 
        & 47.4 & 69.6 & 48.9 & 78.6
        & 12.4 & 27.7 & 36.2 & 58.6 \\

         & OPD (reverse KL)
        & 54.4 & 70.0 & 59.6 & \underline{81.8}
        & 15.7 & 28.5 & 43.2 & 60.1 \\

         & EOPD 
        & \textbf{55.9} & \textbf{72.1} & 58.0 & 81.3
        & 16.1 & \underline{29.7} & 43.3 & \underline{61.0} \\

         & Uni-OPD 
        & 53.1 & 69.3 & \underline{60.5} & 78.8
        & 16.3 & 28.9 & 43.3 & 59.0 \\

         & SFT on teacher traces
        & 48.7 & 69.6 & 50.6 & 79.2
        & 13.6 & 26.9 & 37.6 & 58.6 \\

         & GRPO
        & \underline{55.5} & 65.6 & 55.8 & 70.2
        & \textbf{19.0} & \textbf{31.4} & \underline{43.4} & 55.7 \\

        \rowcolor{methodblue}
         & \textbf{\method}
        & \underline{55.5} & \underline{70.4} & \textbf{63.2} & \textbf{84.1}
        & \underline{16.9} & 29.5 & \textbf{45.2} & \textbf{61.4} \\

        \midrule
        \rowcolor{teachergray}
         & \textit{Teacher 8B}
        & 70.8 & 81.6 & 80.0 & 90.2
        & 28.1 & 43.1 & 59.6 & 71.7 \\

        \bottomrule
    \end{tabular}}
    \caption{\small \textbf{Code generation results.} We report avg@8 and pass@8 on MBPP+, HumanEval+, and LiveCodeBench. \textbf{Bold} and \underline{underlined} values indicate the best and second-best student results, respectively.}
    \label{tab:code}
\end{table}

\subsection{Ablation Studies}
\label{sec:ablation}
To validate the design of \method, we conduct controlled ablations against matched baselines, isolating the effects of difficulty gating, bounded teacher guidance, rollout group size, and teacher strength.

\textbf{Difficulty gating.}
Setting $\gamma=0$ removes difficulty weighting and yields a flat GRPO+OPD objective. Across all four settings, difficulty gating improves avg@8 by $1.3$--$2.4$ points, with gains of $2.4$ points at both mathematics model sizes and $1.3$--$1.4$ points on code (Table~\ref{tab:gating}). Teacher guidance is therefore more effective when concentrated on harder groups than when applied uniformly, and most of what the gate contributes is on the domain where the reward signal becomes dense early.

\textbf{Bounded teacher guidance.} Replacing $2\tanh(\Delta_{i,t}/2)$ with the raw OPD log-probability gap lowers mathematics avg@8 by $1.0$--$1.2$ points and pass@8 by $1.1$--$1.3$ points, while lowering code avg@8 by $1.1$--$1.5$ points and pass@8 by $0.8$--$1.1$ points. The raw signal also increases median gradient norms by $1.4$--$3.3\times$ and produces substantially larger gradient spikes during training (Table~\ref{tab:bound}). Bounding therefore limits extreme teacher updates, stabilizes optimization, and improves held-out accuracy.

\textbf{Smaller rollout groups strengthen the benefit of teacher guidance.} Reducing the group size from $n=16$ to $n=4$ makes reward-uniform groups more frequent during training. GRPO drops by $1.9$ avg@8 points on average across the four settings, compared with only $1.2$ for \method. On 0.6B mathematics, \method moves from $0.5$ points below GRPO at $n=16$ to $1.2$ points above it at $n=4$ (Table~\ref{tab:n4}; Appendix~\ref{app:group_size}). These results are consistent with the intended role of the teacher term, suggesting that it can be especially helpful when group-relative reward feedback is less informative.

\textbf{The method is not highly sensitive to the exact gate strength.} Sweeps over $\lambda$ and $\gamma$ on the 1.7B student show that relatively small teacher weights are sufficient in practice (see Appendix~\ref{app:hpsweep}, Figure~\ref{fig:hpsweep}). Mathematics is nearly flat for $\lambda\in\{0.25,0.5,1\}$ and degrades for larger values, while code peaks near $\lambda=0.5$. Removing difficulty weighting clearly hurts mathematics, while its effect on code is smaller, whereas varying $\gamma\in\{1,2,3\}$ produces only small, non-monotonic changes across settings. Overall, whether teacher guidance is gated matters more than the precise gate sharpness.

\textbf{Why is naive GRPO+OPD insufficient?}
Appendix Table~\ref{tab:gating} compares \method\ with a matched additive GRPO+OPD baseline that uses the same bounded teacher signal on failed trajectories but removes difficulty weighting by setting $\gamma=0$. On mathematics, difficulty gating improves avg@8 by $2.4$ points at both model sizes and pass@8 by $2.8$ and $4.1$ points for the 0.6B and 1.7B students, respectively. On code, the gains are smaller but remain positive: avg@8 improves by $1.3$ and $1.4$ points, while pass@8 improves by $0.1$ and $1.4$ points. These results indicate that teacher guidance alone provides much of the benefit on code, whereas adapting its strength to group difficulty is substantially more important on mathematics. Additional selector and additive controls in Appendix~\ref{app:selector} further separate failed-only routing, difficulty weighting, and bounding. In particular, the comparison against all-rollout bounded and raw GRPO+OPD objectives shows that applying teacher guidance uniformly is not sufficient to reproduce \method's gains, especially on code.

\begin{figure}[t]
    \centering
    \includegraphics[width=\linewidth]{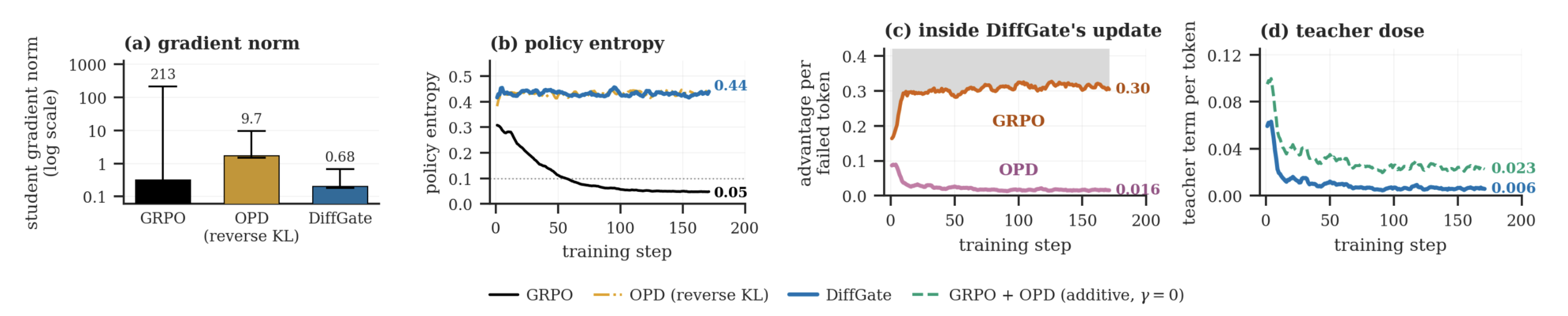}
    \vspace{-0.6cm}
    \caption{\small \textbf{Why \method\ helps. Maths, 1.7B student; all methods use the same data, teacher, and step budget.} \emph{(a)} Student gradient norm over training: the bar denotes the median and the whisker spans minimum to maximum, shown on a log scale. \emph{(b)} Policy entropy, shown as a five-step rolling mean; the dotted line marks $0.1$. \emph{(c)} The two components of \method's update on failed rollouts: the applied teacher term per token (dark blue) and the magnitude of the GRPO advantage per token (light blue). The GRPO advantage is zero for an all-failure group and, with $n=16$, is approximately $0.25$ or larger for groups with $1$--$8$ successes and approximately $1.10$ or larger for groups with $9$ or more successes, ignoring the small $\varepsilon$ term; values are weighted by the measured frequency of each group type. The grey region indicates this approximate lower-bound contribution. \emph{(d)} Applied teacher term per response token under difficulty-gated and additive weighting on the same rollouts.}
    \label{fig:why_diffgate}
\end{figure}

\subsection{Why Does \method Help?}
\label{sec:analysis}
We examine the training dynamics to understand when teacher guidance is most useful and how it interacts with the GRPO signal. We focus on reward-uniform groups, teacher--student disagreement, policy entropy, and gradient behavior (Figure~\ref{fig:why_diffgate}; Appendix Sections~\ref{app:adv_dynamics}--\ref{app:one_view}).

\textbf{Teacher guidance remains available when GRPO has little or no signal.} Early in training, $14$--$36\%$ of prompts form all-failure groups, for which the group-relative GRPO advantage is exactly zero (Table~\ref{tab:adv_dynamics}). These groups quickly become rare on mathematics but persist on code; for example, about $20\%$ of 0.6B code prompts remain all-failure near the end of training. \method continues to provide bounded token-level supervision on these failed trajectories. Accordingly, on 0.6B code, the teacher--student discrepancy on failed rollouts decreases from $1.04$ to $0.37$ nats/token under \method, compared with $0.56$ under GRPO; at 1.7B, it falls to $0.19$ versus $0.30$. On mathematics, GRPO alone already closes much of this gap. This is consistent with the larger gains observed on code, where the reward blind spot persists longer. Gating is what keeps the teacher there: on all-fail groups the gate is one and \method\ coincides with the additive GRPO$+$OPD objective; as the blind spot closes, the gate turns the teacher down $10\times$, from $0.061$ to $0.006$ per token, while the additive rule still applies four times that at the end, most of it on mixed-success groups where GRPO already has a nonzero signal (see Figure~\ref{fig:why_diffgate}c,d), and it trails \method\ (see Table~\ref{tab:gating}).

\textbf{\method exhibits more stable gradient behavior and higher policy entropy.} This effect is clear for the 1.7B student. Late in training, GRPO entropy falls below $0.1$, coinciding with student gradient spikes as large as $213$ on mathematics and $65$ on code. In contrast, \method maintains higher policy entropy and limits the corresponding peaks to $0.7$ and $2.6$, respectively (Figure~\ref{fig:why_diffgate}a,b; Table~\ref{tab:opd_side}). Bounding the teacher signal is important for this stability: replacing $2\tanh(\Delta/2)$ with the raw log-probability gap increases median gradient norms by $1.4$--$3.3\times$ and, on 0.6B code, increases the maximum gradient norm from $4.3$ to $342$. Thus, the bounded teacher term supplies useful local guidance without allowing large teacher--student disagreements in optimization.

\begin{wrapfigure}{r}{0.44\columnwidth}
    \vspace{-0.3cm}
    \centering
    \includegraphics[width=\linewidth]{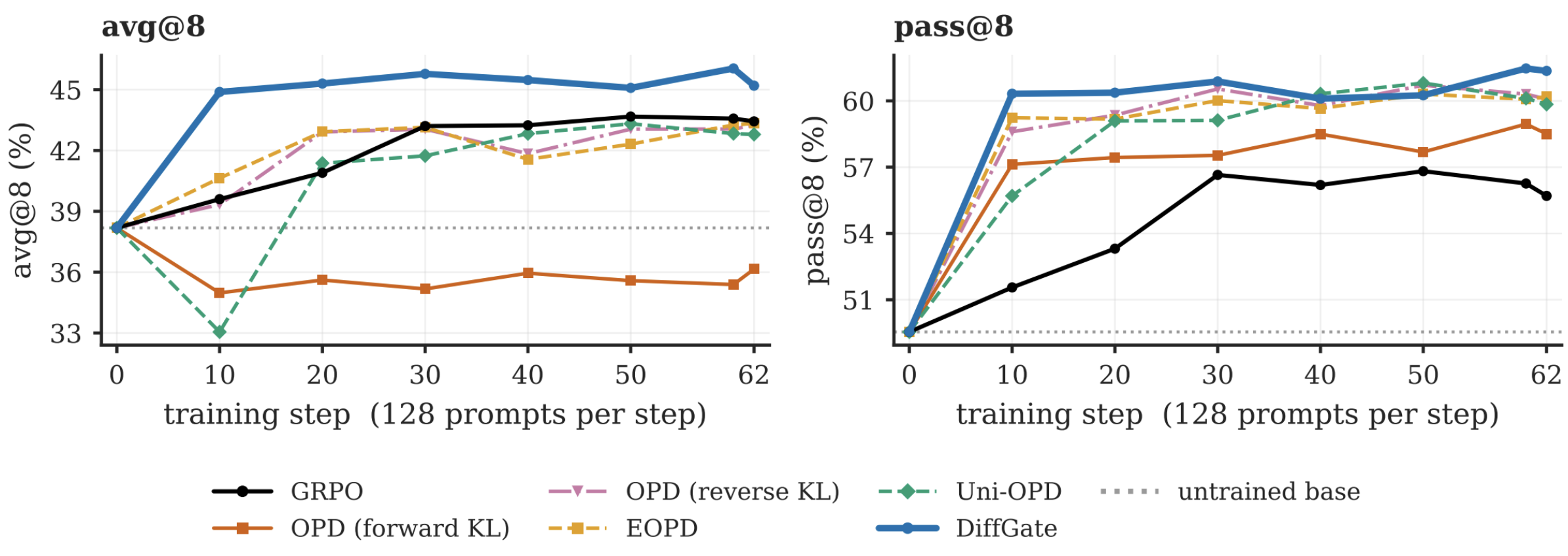}
    \vspace{-0.4cm}
    \caption{\small \textbf{Performance over training.} avg@8 (top) and pass@8 (bottom) across training checkpoints for the 1.7B student on the code benchmarks in Table~\ref{tab:code}. Step~$0$ denotes the base model (dotted line), enabling comparison of how each method improves accuracy and solution coverage.}
    \label{fig:accuracy_curve}
    \vspace{-0.4cm}
\end{wrapfigure}

\textbf{Performance over training.} On the 1.7B code benchmarks, \method\ shows an early and persistent advantage in solution coverage. In the matched checkpoint, DiffGate has higher pass@8 than GRPO at the evaluated code checkpoints, while reaching most of its avg@8 improvement within the first few updates. Compared with OPD baselines, \method\ also improves more rapidly and reaches stronger final performance, suggesting that combining outcome supervision with difficulty-gated teacher guidance is more effective than distillation alone. Overall, the gains emerge early and remain consistent across training rather than appearing only at the final checkpoint (Figure~\ref{fig:accuracy_curve}; Appendix ~\ref{app:stepwise}).

\textbf{Preserving policy diversity is consistent with improved solution coverage.} On mathematics, GRPO achieves a higher training solve rate but converges to a substantially sharper policy, consistent with prior observations of entropy collapse during RLVR training~\citep{wen2026reinforcement}. This advantage largely disappears on held-out avg@8, whereas \method improves pass@8. The higher entropy under \method is therefore consistent with retaining a broader set of plausible solution modes, increasing the chance of solving a problem across multiple samples. This entropy gap appears in all four settings (see Figure~\ref{fig:why_diffgate}b; Figure~\ref{fig:reward_dynamics}), alongside higher pass@8 (see Tables~\ref{tab:math} and~\ref{tab:code}). 

\section{Conclusion and Limitations}
We identify a local--global mismatch between OPD and RLVR: OPD provides dense token-level guidance but is only weakly tied to trajectory correctness, while RLVR captures task success with coarse credit assignment. We introduce \method, which integrates bounded, difficulty-gated teacher guidance into GRPO, concentrating supervision on failed and difficult trajectories. Across Qwen3 students, \method improves solution coverage and shows more robust behavior than OPD and GRPO. 

\textbf{Limitations.} We evaluate \method\ on mathematical reasoning and code generation across three student model sizes. While the results demonstrate its applicability to these settings, its effectiveness on other tasks, such as agentic reasoning and long-horizon decision-making, remains unexplored. Computational constraints limit our evaluation to students of up to 4B parameters.



\bibliography{iclr2027_conference}
\bibliographystyle{iclr2027_conference}


\newpage

\etocdepthtag.toc{mtmain}

\appendix
\etocdepthtag.toc{mtappendix}

\etocsettagdepth{mtmain}{none}
\etocsettagdepth{mtappendix}{subsection}

\etocsettocstyle{\section*{Appendix Contents}\vspace{0.5em}}{}

\onecolumn
\newpage
\tableofcontents
\newpage

\section{Notation and Conventions}
Table~\ref{tab:notations} provides a compact reference for the notation used throughout the paper, covering the policy setup, outcome-based supervision, teacher guidance, and the key components of \method.

\begin{table*}[h]
    \centering
    \small
    \renewcommand{\arraystretch}{1.12}
    \setlength{\tabcolsep}{6pt}
    \caption{\textbf{Summary of notation.}
    Symbols used for the policy setup, outcome supervision,
    teacher guidance, and \method\ optimization.}
    \label{tab:notations}
    \vspace{-0.2cm}

    \begin{tabular}{p{0.22\textwidth} p{0.72\textwidth}}
        \toprule
        \rowcolor{headergray}
        \textbf{Notation} & \textbf{Description} \\
        \midrule

        \rowcolor{sectiongray}
        \multicolumn{2}{l}{\textbf{General setup}} \\
        \midrule

        $\mathcal D$
        & Distribution over training prompts. \\

        $x$
        & Prompt sampled from $\mathcal D$. \\

        $\pi_\theta$
        & Student policy parameterized by $\theta$. \\

        $\pi_{\theta_{\mathrm{old}}}$
        & Frozen behavior policy used to generate trajectories
        for the current policy update. \\

        $\pi_{\mathrm T}$
        & Fixed teacher policy. \\

        $y^{(i)}=(y_1^{(i)},\ldots,y_{L_i}^{(i)})$
        & $i$-th student-generated trajectory of length $L_i$. \\

        $c_{i,t}=(x,y_{<t}^{(i)})$
        & Prefix context preceding token $y_t^{(i)}$
        in trajectory $i$. \\

        $n$
        & Number of trajectories sampled per prompt. \\

        $\operatorname{sg}[\cdot]$
        & Stop-gradient operator; its argument is treated as
        constant during differentiation. \\

        \midrule

        \rowcolor{sectiongray}
        \multicolumn{2}{l}{\textbf{Trajectory-level outcome supervision}} \\
        \midrule

        $R_i=R(x,y^{(i)})$
        & Binary verifier reward for trajectory $i$:
        $R_i=1$ if the response passes the verifier and
        $R_i=0$ otherwise. \\

        $\bar R$
        & Mean reward within a rollout group,
        $\displaystyle
        \bar R=\frac{1}{n}\sum_{i=1}^{n}R_i$. \\

        $\sigma_R$
        & Sample standard deviation of rewards within the rollout group,
        $\displaystyle
        \sigma_R=
        \sqrt{
        \frac{1}{n-1}
        \sum_{i=1}^{n}(R_i-\bar R)^2
        }$. \\

        $p$
        & Group solve rate,
        $\displaystyle
        p=\frac{1}{n}\sum_{i=1}^{n}R_i$.
        For binary rewards, $p=\bar R$. \\

        $A_i^{\mathrm{GRPO}}$
        & Group-relative trajectory advantage,
        $\displaystyle
        A_i^{\mathrm{GRPO}}
        =
        \frac{R_i-\bar R}{\sigma_R+\varepsilon}$. \\

        $\varepsilon$
        & Positive constant used to stabilize reward normalization
        when $\sigma_R$ is small or zero. \\

        \midrule

        \rowcolor{sectiongray}
        \multicolumn{2}{l}{\textbf{Token-level teacher guidance}} \\
        \midrule

        $\Delta_{i,t}$
        & Sampled-token teacher--student log-probability gap,
        $\displaystyle
        \Delta_{i,t}
        =
        \log\pi_{\mathrm T}(y_t^{(i)}\mid c_{i,t})
        -
        \log\pi_{\theta_{\mathrm{old}}}
        (y_t^{(i)}\mid c_{i,t})$. \\

        $B_{i,t}$
        & Bounded teacher signal,
        $\displaystyle
        B_{i,t}
        =
        2\tanh\!\left(
            \frac{\Delta_{i,t}}{2}
        \right)
        \in(-2,2)$. \\

        \midrule

        \rowcolor{sectiongray}
        \multicolumn{2}{l}{\textbf{\method\ components and optimization}} \\
        \midrule

        $w_i$
        & Trajectory-level teacher gate,
        $\displaystyle
        w_i=(1-p)^\gamma(1-R_i)$,
        which activates teacher guidance only on failed trajectories
        and scales it by group difficulty. \\

        $\gamma$
        & Nonnegative exponent controlling the dependence of
        teacher guidance on group difficulty. Setting $\gamma=0$
        removes difficulty weighting while retaining
        failed-trajectory selection. \\

        $\lambda$
        & Nonnegative coefficient controlling the overall strength
        of teacher guidance. Setting $\lambda=0$ recovers GRPO. \\

        $G_{i,t}$
        & Combined token-level advantage used by \method,
        $\displaystyle
        G_{i,t}
        =
        A_i^{\mathrm{GRPO}}
        +
        \lambda w_i B_{i,t}$. \\

        $\rho_{i,t}(\theta)$
        & Token-level importance ratio,
        $\displaystyle
        \rho_{i,t}(\theta)
        =
        \frac{
            \pi_\theta(y_t^{(i)}\mid c_{i,t})
        }{
            \pi_{\theta_{\mathrm{old}}}
            (y_t^{(i)}\mid c_{i,t})
        }$. \\

        $\epsilon_{\mathrm{clip}}$
        & PPO clipping parameter used in the clipped policy surrogate. \\

        \bottomrule
    \end{tabular}
\end{table*}

\section{Failure Modes of On-Policy Distillation}
\label{sec:opd_failure}

\textbf{Analysis setup.} We next examine whether teacher--student disagreement can itself indicate when distillation is needed. Using a Qwen3-1.7B student and a frozen Qwen3-8B teacher, we evaluated rollouts from MATH500, DeepMath, and MBPP+. For each trajectory, we measured the average token-level teacher disagreement and compared it with verifier correctness using AUROC, where $0.5$ indicates chance-level discrimination.

\textbf{Teacher disagreement is largely correctness-blind.} On mathematics, correct and incorrect trajectories show almost identical levels of disagreement. On MATH500, the averages are $0.150$ and $0.148$ nats/token, giving an AUROC of $0.50$; on DeepMath, they are $0.188$ and $0.184$, with AUROC $0.49$. Thus, the magnitude of the local teacher signal provides essentially no useful information about whether a mathematical trajectory is correct. On MBPP+, incorrect trajectories have somewhat larger disagreement than correct ones ($0.551$ versus $0.478$ nats/token), but the signal remains only moderately predictive, with AUROC $0.60$.

\textbf{Large disagreement does not necessarily indicate the reasoning error.} Aggregate disagreement can appear more predictive, but it is also strongly affected by response length: longer trajectories naturally accumulate more disagreement, and response length alone is already predictive of correctness. Moreover, the largest teacher--student discrepancies frequently occur on ordinary words or identifiers rather than on tokens directly associated with the reasoning failure. Hence, a large local disagreement does not reliably tell us either \emph{whether} the trajectory is wrong or \emph{where} the error occurs.

\textbf{Implication for \method.} These observations motivate explicitly separating two complementary roles. The verifier provides trajectory-level information about \emph{when} correction is needed, while the teacher provides dense token-level information about \emph{how} the student should be updated. \method therefore activates teacher guidance only on failed trajectories, where additional corrective supervision is most useful, rather than applying the teacher signal uniformly to both correct and incorrect rollouts.

\section{Theoretical Analysis}
\label{app:theory}

We establish three properties of \method. First, the bounded teacher
signal admits an exact fixed-prefix policy-gradient interpretation through
an $f$-divergence. Second, the difficulty gate increases the relative
importance of teacher guidance on harder groups and supplies a learning
signal on all-failure groups where GRPO is silent. Third, both the teacher
contribution and the resulting token-level advantage are bounded. These
results characterize the learning signal used by \method; they do not
establish convergence of the complete clipped PPO objective.

\subsection{Bounded Teacher Guidance as a Policy Gradient}
\label{sec:theory_bounded_pg}

\textbf{Problem Setup.}
Consider a fixed rollout prefix $c$ and a finite token vocabulary
$\mathcal V$. Let
\begin{align}
    p_{\mathrm{old}}(a)
    &\coloneqq
    \pi_{\theta_{\mathrm{old}}}(a\mid c),
    &
    p_\theta(a)
    &\coloneqq
    \pi_\theta(a\mid c),
    &
    q(a)
    &\coloneqq
    \pi_{\mathrm T}(a\mid c),
    \qquad a\in\mathcal V,
\end{align}
where $p_{\mathrm{old}}$ is the frozen behavior policy, $p_\theta$ is the student being optimized, and $q$ is the fixed teacher. We assume positive probability on every token. Consistent with Equation~\eqref{eq:local_gap}, define
\begin{align}
    \Delta_{\mathrm{old}}(a)
    &\coloneqq
    \log q(a)-\log p_{\mathrm{old}}(a),
\end{align}
and its bounded transformation
\begin{align}
    B_{\mathrm{old}}(a)
    &\coloneqq
    2\tanh\left(
        \frac{\Delta_{\mathrm{old}}(a)}{2}
    \right)
    =
    2\frac{q(a)-p_{\mathrm{old}}(a)}
           {q(a)+p_{\mathrm{old}}(a)}.
    \label{eq:theory_bounded_signal}
\end{align}
Hence $|B_{\mathrm{old}}(a)|<2$, and its sign agrees with the
teacher--student log-probability gap.

\begin{proposition}[Fixed-prefix gradient identity]
\label{prop:bounded_pg_identity}
Define
\begin{align}
    f(r)
    &\coloneqq
    2(r-1)
    -
    4\log\left(\frac{r+1}{2}\right),
    \qquad r>0,
    \label{eq:theory_f_generator}
\end{align}
and the corresponding $f$-divergence
\begin{align}
    D_f(p_\theta\Vert q)
    &\coloneqq
    \sum_{a\in\mathcal V}
    q(a)
    f\left(\frac{p_\theta(a)}{q(a)}\right).
    \label{eq:theory_f_divergence}
\end{align}
Consider the detached sampled-token surrogate
\begin{align}
    \mathcal L_{\mathrm{BPG},c}(\theta)
    &\coloneqq
    -
    \mathbb E_{a\sim p_{\mathrm{old}}}
    \left[
        \operatorname{sg}[B_{\mathrm{old}}(a)]
        \log p_\theta(a)
    \right].
    \label{eq:bounded_pg_surrogate}
\end{align}
Then $f$ is convex, $f(1)=0$, and
\begin{align}
    \left.
    \nabla_\theta
    \mathcal L_{\mathrm{BPG},c}(\theta)
    \right|_{\theta=\theta_{\mathrm{old}}}
    =
    \left.
    \nabla_\theta
    D_f(p_\theta\Vert q)
    \right|_{\theta=\theta_{\mathrm{old}}}.
    \label{eq:bounded_pg_identity}
\end{align}
\end{proposition}

\begin{proof}
Differentiating $f$ gives
\begin{align}
    f'(r)
    &=
    2\frac{r-1}{r+1},
    &
    f''(r)
    &=
    \frac{4}{(r+1)^2}>0.
\end{align}
Thus $f$ is convex, and direct substitution gives $f(1)=0$.

Using Eq.~\eqref{eq:theory_bounded_signal},
\begin{align}
    B_{\mathrm{old}}(a)
    &=
    -f'\left(
        \frac{p_{\mathrm{old}}(a)}{q(a)}
    \right).
    \label{eq:bounded_f_derivative}
\end{align}
Since $q$ is fixed,
\begin{align}
    \nabla_\theta D_f(p_\theta\Vert q)
    &=
    \sum_{a\in\mathcal V}
    f'\left(
        \frac{p_\theta(a)}{q(a)}
    \right)
    \nabla_\theta p_\theta(a)
    \\
    &=
    \mathbb E_{a\sim p_\theta}
    \left[
        f'\left(
            \frac{p_\theta(a)}{q(a)}
        \right)
        \nabla_\theta\log p_\theta(a)
    \right].
\end{align}
Evaluating at $\theta=\theta_{\mathrm{old}}$ and applying
Eq.~\eqref{eq:bounded_f_derivative},
\begin{align}
    \left.
    \nabla_\theta D_f(p_\theta\Vert q)
    \right|_{\theta=\theta_{\mathrm{old}}}
    &=
    -
    \mathbb E_{a\sim p_{\mathrm{old}}}
    \left[
        B_{\mathrm{old}}(a)
        \left.
        \nabla_\theta\log p_\theta(a)
        \right|_{\theta=\theta_{\mathrm{old}}}
    \right].
\end{align}
This is exactly the gradient of
Eq.~\eqref{eq:bounded_pg_surrogate} at
$\theta=\theta_{\mathrm{old}}$, proving the result.
\end{proof}

\textbf{Interpretation.} Proposition~\ref{prop:bounded_pg_identity} gives a local interpretation of the bounded teacher signal: at a fixed prefix and at the behavior policy, minimizing the detached policy-gradient surrogate follows the
same first-order direction as minimizing
$D_f(p_\theta\Vert q)$.

The divergence also has the equivalent form
\begin{align}
    D_f(p_\theta\Vert q)
    &=
    4D_{\mathrm{KL}}
    \left(
        q
        \,\middle\Vert\,
        \frac{p_\theta+q}{2}
    \right).
    \label{eq:theory_mixture_kl}
\end{align}
Indeed, the linear term in $f$ vanishes after summing over the two normalized distributions. Thus, the bounded teacher signal corresponds locally to a valid divergence while avoiding an unbounded sampled-token log-probability ratio.

\textbf{Connection to reverse-KL OPD.}
At the same fixed prefix,
\begin{align}
    \left.
    \nabla_\theta
    D_{\mathrm{KL}}(p_\theta\Vert q)
    \right|_{\theta=\theta_{\mathrm{old}}}
    =
    -
    \mathbb E_{a\sim p_{\mathrm{old}}}
    \left[
        \Delta_{\mathrm{old}}(a)
        \left.
        \nabla_\theta\log p_\theta(a)
        \right|_{\theta=\theta_{\mathrm{old}}}
    \right],
    \label{eq:rkl_gradient_comparison}
\end{align}
where the additive constant from differentiating the reverse KL vanishes by the score-function identity. Moreover,
\begin{align}
    2\tanh\left(\frac{z}{2}\right)
    =
    z-\frac{z^3}{12}+O(z^5),
    \qquad z\rightarrow0.
    \label{eq:bounded_taylor}
\end{align}
Hence the bounded and standard OPD signals agree to first order for small teacher--student gaps, while the bounded signal saturates for large discrepancies.

\begin{remark}[Relation to the implemented objective]
\label{rem:bounded_pg_scope}
The identity above is local: it holds at a fixed prefix and at
$\theta=\theta_{\mathrm{old}}$. In \method,
$B_{i,t}$ is evaluated only on sampled student tokens, exactly as in
Eq.~\eqref{eq:bounded_advantage}; the full-vocabulary divergence in
Eq.~\eqref{eq:theory_f_divergence} is never computed.

The implemented method additionally multiplies $B_{i,t}$ by the trajectory-dependent gate $w_i$, adds the GRPO advantage, and optimizes the resulting detached signal using PPO clipping. Therefore, Proposition~\ref{prop:bounded_pg_identity} provides a fixed-prefix interpretation of the teacher component only; it does not imply that the complete \method\ objective globally minimizes Eq.~\eqref{eq:theory_f_divergence}.
\end{remark}

\subsection{Difficulty Gating and Relative Supervision Strength}
\label{sec:theory_relative_strength}

Consider a group of $n\geq2$ trajectories with binary rewards
$R_i\in\{0,1\}$ and group solve rate
\begin{align}
    p
    &=
    \frac{1}{n}
    \sum_{i=1}^{n}R_i.
\end{align}
For binary rewards, $p=\bar R$. Since we use the sample standard
deviation defined in Section~\ref{sec:preliminaries},
\begin{align}
    \sigma_R
    &=
    \sqrt{
        \frac{n}{n-1}p(1-p)
    }
    =
    \sqrt{\kappa_n p(1-p)},
    \qquad
    \kappa_n
    \coloneqq
    \frac{n}{n-1}.
    \label{eq:theory_binary_std}
\end{align}
The GRPO advantage, difficulty gate, and combined \method\ signal are therefore
\begin{align}
    A_i^{\mathrm{GRPO}}
    &=
    \frac{R_i-p}
         {\sqrt{\kappa_n p(1-p)}+\varepsilon},
    \label{eq:theory_grpo_adv}
    \\
    w_i
    &=
    (1-p)^\gamma(1-R_i),
    \label{eq:theory_gate}
    \\
    G_{i,t}
    &=
    A_i^{\mathrm{GRPO}}
    +
    \lambda w_i B_{i,t}.
    \label{eq:theory_combined_advantage}
\end{align}

For a failed trajectory, $R_i=0$, these reduce to
\begin{align}
    A_i^{\mathrm{GRPO}}
    &=
    -\frac{p}
           {\sqrt{\kappa_n p(1-p)}+\varepsilon},
    &
    w_i
    &=
    (1-p)^\gamma.
    \label{eq:failed_group_signals}
\end{align}

\begin{proposition}[Difficulty-dependent relative teacher strength]
\label{prop:relative_teacher_strength}
For a failed trajectory in a mixed group, $0<p<1$, the magnitude of
the teacher contribution relative to the GRPO advantage satisfies
\begin{align}
    \frac{
        |\lambda w_i B_{i,t}|
    }{
        |A_i^{\mathrm{GRPO}}|
    }
    &\leq
    2\lambda
    (1-p)^\gamma
    \frac{
        \sqrt{\kappa_n p(1-p)}+\varepsilon
    }{p}
    \eqqcolon
    U_n(p).
    \label{eq:relative_teacher_bound}
\end{align}
For fixed $\lambda\geq0$, $\gamma\geq0$, $\varepsilon>0$, and $n\geq2$,
$U_n(p)$ is non-increasing on $p\in(0,1)$.
\end{proposition}

\begin{proof}
For a failed trajectory,
$w_i=(1-p)^\gamma$. Since $|B_{i,t}|<2$,
\begin{align}
    |\lambda w_i B_{i,t}|
    &<
    2\lambda(1-p)^\gamma.
\end{align}
Moreover,
\begin{align}
    |A_i^{\mathrm{GRPO}}|
    &=
    \frac{p}
         {\sqrt{\kappa_n p(1-p)}+\varepsilon}.
\end{align}
Combining the two expressions gives
Eq.~\eqref{eq:relative_teacher_bound}.

Expanding the upper bound yields
\begin{align}
    U_n(p)
    =
    2\lambda
    \left[
        \sqrt{\kappa_n}\,
        \frac{(1-p)^{\gamma+1/2}}{\sqrt{p}}
        +
        \varepsilon
        \frac{(1-p)^\gamma}{p}
    \right].
    \label{eq:relative_teacher_expansion}
\end{align}
For $\gamma\geq0$, both terms are non-increasing on
$p\in(0,1)$; therefore, their nonnegative sum is also non-increasing.
\end{proof}

\textbf{Boundary behavior.}
When $p=0$, every rollout fails. The GRPO advantages vanish and $w_i=1$, giving the following advantage:
\begin{align}
    G_{i,t}
    =
    \lambda B_{i,t}.
    \label{eq:theory_all_fail}
\end{align}
Thus, the teacher receives its maximum \emph{gate weight} on all-failure groups, where GRPO provides no learning signal. For a mixed group, $0<p<1$, both signals are available on failed trajectories, and Proposition~\ref{prop:relative_teacher_strength}
shows that an upper bound on the teacher contribution relative to GRPO decreases as the group solve rate increases.

At the opposite boundary, $p=1$, every trajectory succeeds. In this case, $A_i^{\mathrm{GRPO}}=0$ and $w_i=0$, so
\begin{align}
    G_{i,t}=0.
\end{align}

The average gate weight across the group also has a simple form:
\begin{align}
    \frac{1}{n}
    \sum_{i=1}^{n}w_i
    &=
    (1-p)^{\gamma+1}.
    \label{eq:theory_average_gate}
\end{align}
Hence, the average teacher weight decreases as the group solve rate increases.

\begin{remark}[Scope of the relative-strength result] \label{rem:relative_strength_scope}
Proposition~\ref{prop:relative_teacher_strength} concerns an upper bound on the relative magnitudes of the two scalar advantage terms. It does not imply that the realized teacher contribution is monotonic at every token, since $|B_{i,t}|$ varies across tokens and policies. It also does not establish agreement between teacher and outcome supervision, improvement in verifier reward, or a bound on the resulting parameter update.
\end{remark}

\subsection{Bounded Combined Token-Level Advantage} \label{sec:theory_advantage_bound}

The bounded teacher transformation also provides direct control over the scale of the scalar learning signal supplied to the policy surrogate.

\begin{proposition}[Bounded combined advantage]
\label{prop:combined_advantage_bound}
For a rollout group of size $n\geq2$ with binary rewards and sample standard-deviation normalization,
\begin{align}
    \left|A_i^{\mathrm{GRPO}}\right|
    &\leq
    \frac{n-1}{\sqrt{n}}.
    \label{eq:grpo_advantage_bound}
\end{align}
Since $0\leq w_i\leq1$ and $|B_{i,t}|<2$, the teacher contribution satisfies
\begin{align}
    |\lambda w_i B_{i,t}|
    &<
    2\lambda,
\end{align}
and therefore
\begin{align}
    |G_{i,t}|
    &<
    \frac{n-1}{\sqrt{n}}
    +
    2\lambda.
    \label{eq:combined_advantage_bound}
\end{align}
\end{proposition}

\begin{proof}
If $p\in\{0,1\}$, all rewards are identical and
$A_i^{\mathrm{GRPO}}=0$. We therefore consider $0<p<1$.

For a successful trajectory, $R_i=1$,
\begin{align}
    |A_i^{\mathrm{GRPO}}|
    &=
    \frac{1-p}
         {\sqrt{\kappa_n p(1-p)}+\varepsilon}
    \\
    &\leq
    \sqrt{
        \frac{1-p}{\kappa_n p}
    }
    =
    \sqrt{
        \frac{n-1}{n}
        \frac{1-p}{p}
    }.
\end{align}
A mixed group contains at least one successful rollout, so $p\geq1/n$. Therefore,
\begin{align}
    |A_i^{\mathrm{GRPO}}|
    &\leq
    \sqrt{
        \frac{n-1}{n}(n-1)
    }
    =
    \frac{n-1}{\sqrt n}.
\end{align}

For a failed trajectory, $R_i=0$,
\begin{align}
    |A_i^{\mathrm{GRPO}}|
    &=
    \frac{p}
         {\sqrt{\kappa_n p(1-p)}+\varepsilon}
    \\
    &\leq
    \sqrt{
        \frac{p}{\kappa_n(1-p)}
    }
    =
    \sqrt{
        \frac{n-1}{n}
        \frac{p}{1-p}
    }.
\end{align}
Because a mixed group contains at least one failed rollout, $p\leq(n-1)/n$, giving
\begin{align}
    |A_i^{\mathrm{GRPO}}|
    &\leq
    \frac{n-1}{\sqrt n}.
\end{align}

Finally, $w_i\leq1$ and $|B_{i,t}|<2$ imply
$|\lambda w_iB_{i,t}|<2\lambda$. Applying the triangle inequality to $G_{i,t}=A_i^{\mathrm{GRPO}}+\lambda w_iB_{i,t}$ completes the proof.
\end{proof}

\begin{remark}[What the bound guarantees]
\label{rem:combined_advantage_scope}
Equation~\eqref{eq:combined_advantage_bound} bounds the scalar token-level advantage supplied to the PPO surrogate. It does not by
itself bound parameter-gradient norms, parameter movement, KL drift between successive policies, or guarantee improvement in task reward. These quantities additionally depend on the policy Jacobian, optimizer, and PPO clipping dynamics.
\end{remark}

\begin{table}[t]
    \centering
    \small
    \renewcommand{\arraystretch}{1.08}
    \setlength{\tabcolsep}{4pt}
    \begin{tabular}{lll}
        \toprule
        \rowcolor{headergray}
        \textbf{Group} & \textbf{Setting} & \textbf{Value} \\
        \midrule

        \multirow{3}{*}{Models}
          & Student & Qwen3-0.6B / Qwen3-1.7B (also Qwen3-4B, App.~\ref{app:scale4b}) \\
          & Teacher & Qwen3-8B, frozen (also Qwen3-14B, App.~\ref{app:scale4b}) \\
          & Mode & Non-thinking; shared chat template \\
        \midrule

        \multirow{5}{*}{Optimization}
          & Optimizer & AdamW \\
          & Learning rate & $3\times10^{-6}$, cosine schedule, no warmup \\
          & Prompts per step & $128$; PPO mini-batch $32$ \\
          & PPO clip $\epsilon$ & $0.2$ \\
          & Entropy bonus & $0$ \\
        \midrule

        \multirow{4}{*}{Rollouts}
          & Group size $n$ & $16$ (\method, GRPO); $1$ (OPD baselines) \\
          & Sampling & Temperature $1.0$, top-$p$ $1.0$ \\
          & Lengths & Prompt $\leq 2048$, response $\leq 4096$ tokens \\
          & Reward & Binary: answer matching (math), unit tests (code) \\
        \midrule

        \multirow{3}{*}{\method}
          & Teacher strength $\lambda$ & $1$ \\
          & Difficulty exponent $\gamma$ & $1$ \\
          & Bounded signal & $2\tanh(\Delta_{i,t}/2)\in(-2,2)$ \\
        \midrule

        \multirow{3}{*}{Regularization}
          & KL in reward & None \\
          & KL in loss & None \\
          & Reference-policy KL & Not used \\
        \midrule

        \multirow{3}{*}{Budget}
          & Mathematics & $171$ steps; $21{,}888$ prompt presentations \\
          & Code & $62$ steps; $7{,}936$ prompt presentations \\
          & Hardware & $1$ node, $4\times$A100-80GB, FSDP \\
        \midrule

        \multirow{2}{*}{Evaluation}
          & Decoding & $16$ generations; $T=0.6$, top-$p=1.0$, repetition penalty $1.15$ \\
          & Length & $\leq 16{,}384$ new tokens; avg@8 and pass@8 from the same pool \\
        \bottomrule
    \end{tabular}
    \vspace{-0.2cm}
    \caption{\small \textbf{Training and evaluation configuration.} Shared across methods unless otherwise specified. Training uses $128$ prompts per step. The mathematics and code training sets contain $21{,}928$ and $4{,}082$ prompts, respectively; the reported budgets correspond to approximately $1.00$ and $1.94$ dataset passes. GRPO-based methods differ only in their teacher-guidance configuration.}
    \label{tab:hparams}
    \vspace{-0.2cm}
\end{table}

\section{Implementation Details}
\label{appendix:implementation_details}
This section presents the experimental setup and implementation details, including the models, datasets, baselines, training and evaluation configurations, computational cost, hyperparameter analysis, and the complete \method\ algorithm.

\subsection{Models, Data and Baselines}
\textbf{Models.} We use Qwen3-0.6B and Qwen3-1.7B students with a frozen Qwen3-8B teacher, and extend the evaluation to Qwen3-4B students with Qwen3-8B and Qwen3-14B teachers in Appendix Section ~\ref{app:scale4b}. All models operate in non-thinking mode with a shared chat template, allowing the teacher to evaluate the student-generated tokens under the same prefix context.

\textbf{Training Data.} We train on \textsc{DeepMath}~\citep{he2026deepmath}, retaining mathematics problems with difficulty greater than $6$, and on \textsc{TACO}~\citep{li2023taco} for code generation. Mathematics responses receive a binary reward based on final-answer matching, while code responses are evaluated using unit tests executed in a sandbox. The training sets contain $21{,}928$ mathematics prompts and $4{,}082$ code prompts. With $128$ prompts per step, we train for $171$ mathematics steps and $62$ code steps, corresponding to $21{,}888$ and $7{,}936$ prompt presentations, respectively.

\textbf{Baselines and Implementation.}
We implement GRPO, \method, the boundedness and gating ablations (Appendix Sections ~\ref{app:bound} and~\ref{app:gating}), and the forward- and reverse-KL OPD baselines using a common training framework built on verl\footnote{\url{https://github.com/volcengine/verl}}. These methods share the data pipeline, rollout engine, verifier, and optimizer. GRPO is recovered by setting $\lambda=0$, while setting $\gamma=0$ yields the additive GRPO$+$OPD baseline without difficulty weighting. Thus, the GRPO-based variants differ only in their teacher-guidance configuration. We train EOPD~\citep{jin2026entropy} and Uni-OPD~\citep{hou2026uni} using their original implementations, with the same training data and student--teacher pairs. The SFT baseline fine-tunes the student on teacher-generated solutions for the same prompts.

\textbf{Training Configuration.} Table~\ref{tab:hparams} summarizes the training and evaluation settings. Within the common implementation, we use the same optimizer, prompt batch size, and training-step budget without method-specific optimizer tuning. GRPO-based methods sample $n=16$ trajectories per prompt, whereas the OPD baselines use a single trajectory. We fix $\lambda=\gamma=1$ for \method across all model sizes and domains, without per-domain tuning. No method uses a reference-policy KL penalty in its reward or training loss.

\textbf{Computational Resources and Software Environment.}
All training and evaluation experiments were conducted on Ubuntu 22.04 Linux nodes, each equipped with four NVIDIA A100 80GB GPUs. Each run reported in Tables~\ref{tab:math} and~\ref{tab:code} used a single node. The software environment comprised Python 3.12, PyTorch 2.8.0\footnote{\url{https://pytorch.org/}} with CUDA 12.9\footnote{\url{https://developer.nvidia.com/cuda-toolkit}}, FlashAttention 2.8.3\footnote{\url{https://github.com/Dao-AILab/flash-attention}}, and Transformers 4.57.1\footnote{\url{https://github.com/huggingface/transformers}}. We used vLLM 0.11.0 \footnote{\url{https://github.com/vllm-project/vllm}} for rollout generation during training and vLLM 0.18.0 for evaluation. Training was performed using our verl-based implementation, while EOPD and Uni-OPD were run using the original implementations released by their respective authors.

\subsection{Algorithm}
\begin{algorithm}[t]
\caption{\method: Difficulty-Gated On-Policy Distillation}
\label{alg:method}
\begin{algorithmic}[1]

\Require Student policy $\pi_\theta$, fixed teacher $\pi_{\mathrm{T}}$, prompt distribution $\mathcal{D}$, verifier $R$, group size $n$, teacher weight $\lambda$, difficulty exponent $\gamma$, PPO clipping parameter $\epsilon$

\For{each training iteration}
    \State Freeze behavior policy:
    $\pi_{\theta_{\mathrm{old}}} \leftarrow \pi_\theta$

    \State Sample a batch of prompts $x \sim \mathcal{D}$

    \For{each prompt $x$}

        \State Sample $n$ on-policy trajectories:
        $
        y^{(i)}
        \sim
        \pi_{\theta_{\mathrm{old}}}(\cdot \mid x),
        \qquad i=1,\ldots,n
        $

        \State Evaluate trajectory rewards:
        $R_i = R(x,y^{(i)}) \in \{0,1\}$

        \State Compute group statistics:
        \[
        p = \frac{1}{n}\sum_{i=1}^{n}R_i,
        \qquad
        A_i^{\mathrm{GRPO}}
        =
        \frac{R_i-\bar R}{\sigma_R+\varepsilon},
        \qquad
        \bar R=\frac{1}{n}\sum_{j=1}^{n}R_j
        \]

        \For{each trajectory $i$ and token $t$}

            \State Define context:
            $c_{i,t}=(x,y^{(i)}_{<t})$

            \State Compute teacher--student discrepancy:
            \[
            \Delta_{i,t}
            =
            \log\pi_{\mathrm{T}}
            (y_t^{(i)}\mid c_{i,t})
            -
            \log\pi_{\theta_{\mathrm{old}}}
            (y_t^{(i)}\mid c_{i,t})
            \]
            \State Bound the local teacher signal:
            $
            B_{i,t}
            =
            2\tanh\!\left(\frac{\Delta_{i,t}}{2}\right)
            $

            \State Compute difficulty gate:
            $
            w_i=(1-p)^\gamma(1-R_i)
            $

            \State Construct difficulty-gated token advantage:
            $G_{i,t}
            =
            A_i^{\mathrm{GRPO}}
            +
            \lambda w_i B_{i,t}$

        \EndFor
    \EndFor

    \State Compute token-level importance ratios:$
    \rho_{i,t}(\theta)
    =
    \frac{
        \pi_\theta(y_t^{(i)}\mid c_{i,t})
    }{
        \pi_{\theta_{\mathrm{old}}}(y_t^{(i)}\mid c_{i,t})
    }$

    \State Update $\theta$ using the clipped surrogate:
    \[
    \mathcal{L}_{\method}
    =
    -\frac{1}{n}
    \sum_{i=1}^{n}
    \frac{1}{L_i}
    \sum_{t=1}^{L_i}
    \min
    \left(
        \rho_{i,t}(\theta)\operatorname{sg}[G_{i,t}],
        \operatorname{clip}
        (\rho_{i,t}(\theta),1-\epsilon_{\mathrm{clip}},1+\epsilon_{\mathrm{clip}})
        \operatorname{sg}[G_{i,t}]
    \right)
    \]

\EndFor
\end{algorithmic}
\end{algorithm}

The verifier determines \emph{which trajectories} receive teacher guidance, while the teacher supplies bounded token-level update directions within the selected failed trajectories. The gate therefore operates at the trajectory level; \method\ does not attempt to localize the causal error token within a failed response. For successful trajectories, $R_i=1$ and hence $w_i=0$, so no teacher update is applied. For all-failure groups, $p=0$ and the GRPO advantage vanishes, making teacher guidance maximal. Setting $\lambda=0$ recovers GRPO, whereas $\gamma=0$ removes difficulty weighting and yields a flat teacher-guidance baseline.

\begin{table}[h]
    \centering
    \small
    \renewcommand{\arraystretch}{1.1}
    \setlength{\tabcolsep}{4pt}
    \begin{tabular}{llcccc}
        \toprule
        \rowcolor{headergray} \textbf{Arm} & \textbf{Rollouts / prompt} & \textbf{Math 0.6B} & \textbf{Math 1.7B} & \textbf{Code 0.6B} & \textbf{Code 1.7B} \\
        \rowcolor{headergray} & & \multicolumn{2}{c}{\textit{171 steps}} & \multicolumn{2}{c}{\textit{62 steps}} \\
        \midrule
        OPD (forward KL) & 1 & 4 & 4.7 & 1.5 & 5 \\
        OPD (reverse KL) & 1 & 4.1 & 4.5 & 1.4 & 1.5 \\
        EOPD & 1 & 4.2 & 4.7 & 4 & 4 \\
        Uni-OPD & 8 & 15 & 17 & 14 & 13 \\
        GRPO & 16 & 36 & 39 & 16 & 19 \\
        \method & 16 & 38 & 37 & 14 & 19 \\
        \bottomrule
    \end{tabular}
    \vspace{-0.2cm}
    \caption{\small \textbf{Training cost.} Wall-clock hours on $4\times$ A100-80GB for the runs behind main Tables~\ref{tab:math} and~\ref{tab:code}.}
    \vspace{-0.2cm}
    \label{tab:cost}
\end{table}

\subsection{Training Cost}\label{app:cost}
Table~\ref{tab:cost} reports wall-clock training time on one node with $4\times$ A100-80GB GPUs. \method\ has comparable computational cost to GRPO, with measured wall-clock times differing by at most $12.5\%$ across the four settings. Standard OPD baselines are generally cheaper because they use fewer rollouts per prompt, while \method\ and GRPO both use grouped on-policy generation.

\subsection{Hyperparameter Sweep} \label{app:hpsweep}
Figure~\ref{fig:hpsweep} varies the teacher strength $\lambda$ and the difficulty exponent $\gamma$ on the 1.7B student. Every point is taken at the same training step within a domain (mathematics $100$, code $62$, the last step all sweep arms share), so the values differ from those of Tables~\ref{tab:math} and~\ref{tab:code}; the no-gate control $\gamma=0$ is reported at its final step, and $\lambda=0.25$ was not run on code. We summarize the key insights as follows:
\begin{itemize}
    \setlength{\itemsep}{1pt}
    \item \textbf{Moderate teacher strength.} Mathematics avg@8 is $21.4$--$21.9$ for
    $\lambda\le1$ and falls to $19.6$ and $19.0$ at $\lambda=2$ and $4$. Code peaks at $\lambda=0.5$
    ($46.2$) and falls to $42.7$ at $\lambda=8$.
    \item \textbf{Gating matters more than its sharpness.} Removing the gate ($\gamma=0$) costs $3.2$ avg@8 against the $\lambda{=}1$ point of this sweep on mathematics ($18.7$ against $21.9$) and is neutral on code ($43.8$ against $43.4$--$45.2$); the matched comparison at the full step budget, in Table~\ref{tab:gating}, puts the same gap at $2.4$. Varying $\gamma\in\{1,2,3\}$ changes avg@8 by at most $1.8$ on mathematics and $2.2$ on code, with no consistent direction.
    \item \textbf{One setting for every cell.} All reported runs use $\lambda=\gamma=1$. The best code setting in the sweep, $\lambda=0.5$, is $1.0$ avg@8 above the reported $\lambda=1$ run
    ($46.2$ against $45.2$) on a single seed; we did not tune per domain.
\end{itemize}

\subsection{Evaluation Protocol}
\label{app:eval_protocol}
This section provides additional details on the evaluation protocol, including metric computation, benchmark composition, answer and code verification, randomness, and evaluation record keeping.

\textbf{Metric computation.} For each test problem, we generate $8$ responses. We report two metrics: avg@8 and pass@8. Avg@8 is the mean correctness across the eight generated responses. Pass@8 measures whether at least one of the eight sampled responses is correct, following the standard pass@$k$ formulation of \citet{chen2021evaluating}. When reporting results averaged across multiple benchmarks, we use the unweighted mean of the corresponding per-benchmark scores.

\textbf{Randomness.} Training uses seed $42$ for data ordering and batching. Evaluation does not explicitly set a sampling seed in vLLM. Unless otherwise stated, each reported model corresponds to a single training run.

\textbf{Benchmark composition.} The mathematics evaluation contains Minerva-Math ($272$ problems), OlympiadBench English open-ended mathematics ($674$), AIME24 ($30$), AIME25 ($30$), and HMMT25 ($30$). Code evaluation uses MBPP+ ($378$), HumanEval+ ($164$), and $454$ code-generation problems from LiveCodeBench. The science evaluation uses GPQA-Diamond ($198$), MMLU-STEM ($3{,}153$ questions across $19$ STEM subjects), and OlympiadBench-Physics ($233$).

\textbf{Mathematics and science verification.}
For mathematics, we first extract an answer from an \texttt{<answer>} tag when available, otherwise from the final \texttt{\textbackslash boxed\{\}} expression, and finally from the last numerical expression in the response. We use Math-Verify\footnote{\url{https://github.com/huggingface/Math-Verify}} to determine equivalence with the reference answer after standard normalization. If symbolic verification fails, normalized strings are compared directly, with numerical answers accepted within a tolerance of $10^{-4}$. For GPQA and MMLU-STEM, we extract the selected option from the answer tag or, if absent, the final stated option.

\textbf{Code verification.} Generated code is extracted from the final fenced Python block or, when no fenced block is present, from the first top-level definition or import onward. Each execution is isolated in a Python subprocess with fixed time and memory limits for safe and reproducible evaluation. HumanEval+ and MBPP+ use the EvalPlus test harnesses, and a response is marked correct only if it passes all tests. LiveCodeBench and TACO use their provided input--output tests; outputs are compared after removing leading and trailing whitespace.

\textbf{Contamination and evaluation records.} DeepMath-103K was decontaminated by its authors against common mathematics benchmarks~\citep{he2026deepmath}. We did not perform an additional decontamination analysis for DeepMath or TACO against the evaluation sets used here. For each checkpoint and benchmark, we store per-problem correctness for all $16$ generations together with the extracted answer and first complete response.

\begin{figure}[t]
    \centering
    \includegraphics[width=0.7\linewidth]{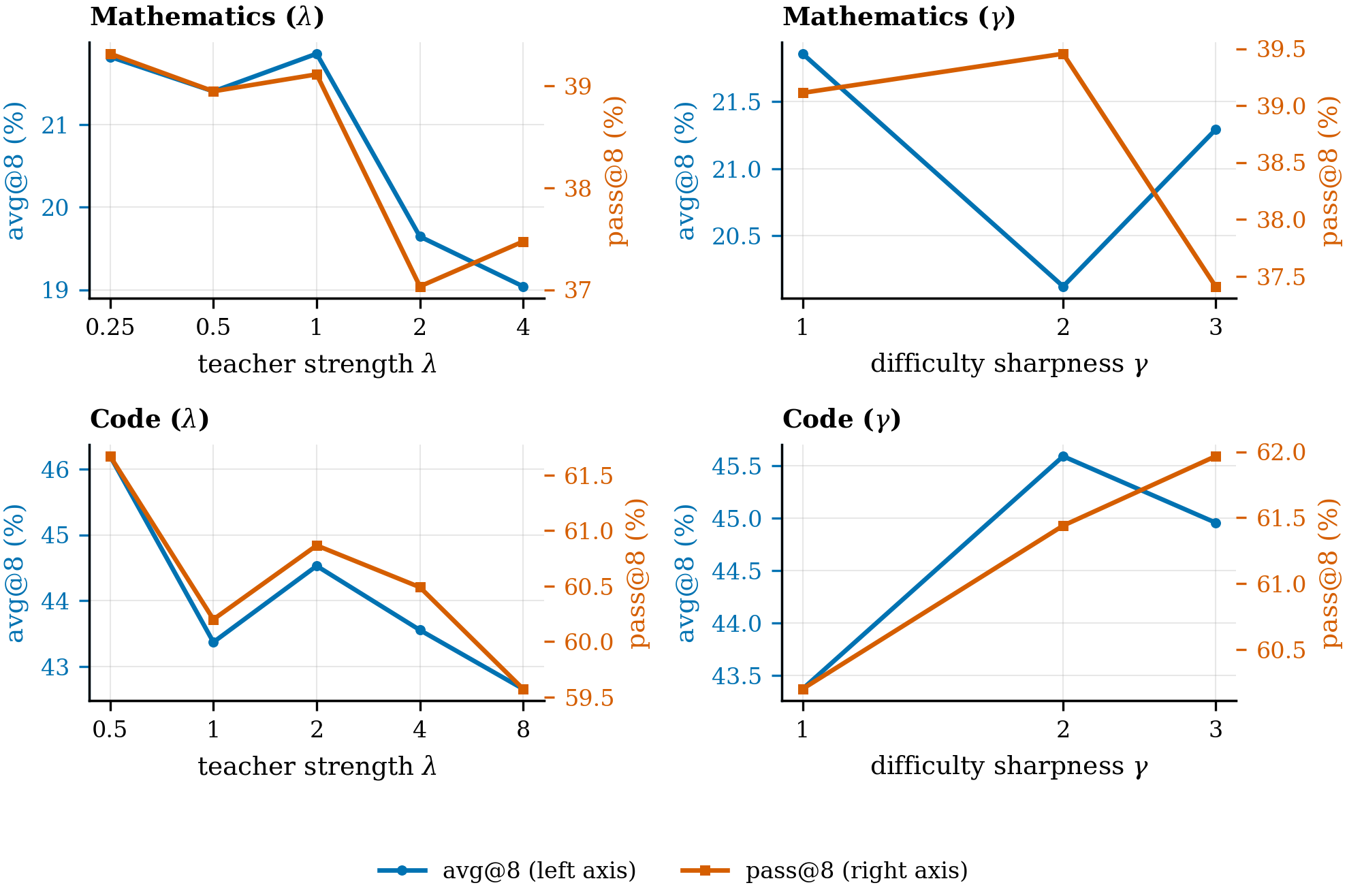}
    \vspace{-0.2cm}
    \caption{\small \textbf{Hyperparameter sweep, 1.7B student.} Teacher strength $\lambda$ (left) and difficulty sharpness $\gamma$ (right), for mathematics (top) and code (bottom); avg@8 on the left axis, pass@8 on the right. Every point is \method\ at a different setting, so this figure varies one constant rather than comparing arms. Hollow markers are the reported configuration $\lambda=\gamma=1$.}
    \label{fig:hpsweep}
    \vspace{-0.2cm}
\end{figure}

\section{Additional Experimental Results}
This section provides supplementary experimental results and analyses that complement the main paper. We report out-of-distribution evaluation on science benchmarks, additional ablations on key design choices, training-dynamics analyses, and rollout-group-size experiments. These results further characterize when \method is most effective and help isolate the contribution of each component.

\subsection{Out-of-Distribution Transfer to Science} \label{app:science}
Table~\ref{tab:science} gives the per-benchmark scores behind the science columns of Table~\ref{tab:math}: mathematics-trained students evaluated on GPQA, MMLU-STEM, and
OlympiadBench-Physics with no science training. The gains over the base student come from GPQA and MMLU-STEM (at 1.7B, \method\ reaches $33.2$ and $74.3$ avg@8 against $14.4$ and $59.2$), while OlympiadBench-Physics stays within two points of the base student for every method.

\begin{table}[t]
    \centering
    \small
    \renewcommand{\arraystretch}{1.13}
    \setlength{\tabcolsep}{4.2pt}

    \resizebox{\textwidth}{!}{%
    \begin{tabular}{lcccccc|cc}
        \toprule

        \rowcolor{headergray}
        \multirow{1}{*}{\textbf{Method}}
        & \multicolumn{2}{c}{\textbf{GPQA}}
        & \multicolumn{2}{c}{\textbf{MMLU-STEM}}
        & \multicolumn{2}{c}{\textbf{OlympiadBench-Physics}}
        & \multicolumn{2}{c}{\textbf{Average}} \\

        \cmidrule(lr){2-3}\cmidrule(lr){4-5}
        \cmidrule(lr){6-7}\cmidrule(lr){8-9}
        \rowcolor{headergray}
        & \textit{avg@8} & \textit{pass@8}
        & \textit{avg@8} & \textit{pass@8}
        & \textit{avg@8} & \textit{pass@8}
        & \textit{avg@8} & \textit{pass@8} \\

        \midrule

        \multicolumn{9}{l}{\textit{Qwen3-0.6B mathematics student}} \\

        Base
        & 13.0 & 50.8
        & 48.5 & 86.2
        & 3.6 & 9.4
        & 21.7 & 48.8 \\

        OPD (forward KL)
        & 26.1 & \textbf{77.3}
        & 58.5 & \textbf{88.5}
        & 3.6 & 11.4
        & 29.4 & \textbf{59.1} \\

        OPD (reverse KL)
        & 26.5 & 74.2
        & 59.0 & 87.8
        & \underline{4.8} & \underline{12.6}
        & 30.1 & 58.2 \\

        EOPD
        & 26.1 & 74.3
        & 59.7 & \underline{88.1}
        & 4.7 & 12.0
        & 30.2 & 58.1 \\

        Uni-OPD
        & \underline{27.0} & 73.2
        & \underline{60.2} & 87.8
        & 4.7 & 12.0
        & \underline{30.6} & 57.7 \\

        GRPO
        & 12.8 & 52.8
        & 50.1 & 85.2
        & \textbf{5.6} & 12.2
        & 22.8 & 50.1 \\

        \rowcolor{methodblue}
        \textbf{\method}
        & \textbf{28.3} & \underline{74.4}
        & \textbf{60.6} & 87.7
        & 4.6 & \textbf{12.7}
        & \textbf{31.2} & \underline{58.3} \\

        \midrule

        \multicolumn{9}{l}{\textit{Qwen3-1.7B mathematics student}} \\

        Base
        & 14.4 & 57.6
        & 59.2 & 90.8
        & 9.0 & 21.1
        & 27.5 & 56.5 \\

        OPD (forward KL)
        & \underline{32.1} & \textbf{78.5}
        & 68.0 & \textbf{92.4}
        & 7.3 & 18.2
        & 35.8 & \textbf{63.0} \\

        OPD (reverse KL)
        & 30.2 & 73.0
        & \underline{70.4} & 91.9
        & 9.3 & \underline{21.5}
        & \underline{36.7} & 62.1 \\

        EOPD
        & 30.7 & \underline{76.4}
        & 68.5 & 92.1
        & 9.3 & 20.6
        & 36.2 & \textbf{63.0} \\

        Uni-OPD
        & 31.0 & 76.2
        & 68.9 & 91.6
        & \underline{9.6} & 21.2
        & 36.5 & \textbf{63.0} \\

        GRPO
        & 7.8 & 29.1
        & 26.7 & 56.9
        & 8.4 & 20.1
        & 14.3 & 35.4 \\

        \rowcolor{methodblue}
        \textbf{\method}
        & \textbf{33.2} & 73.0
        & \textbf{74.3} & \underline{92.2}
        & \textbf{10.0} & \textbf{22.2}
        & \textbf{39.2} & \underline{62.4} \\

        \midrule

        \rowcolor{teachergray}
        \textit{Teacher 8B}
        & 44.8 & 81.7
        & 82.1 & 95.5
        & 15.9 & 31.9
        & 47.6 & 69.7 \\

        \bottomrule
    \end{tabular}}
    \vspace{-0.2cm}
    \caption{\small \textbf{Out-of-distribution transfer to Science.} avg@8 / pass@8 of mathematics-trained students on GPQA, MMLU-STEM and OlympiadBench-Physics, with no science training; the final columns average the three and are the science columns of Table~\ref{tab:math}. \textbf{Bold} / \underline{underline}: best / second-best per student.}
    \vspace{-0.2cm}
    \label{tab:science}
\end{table}

\subsection{Scaling to Qwen3-4B Students}
\label{app:scale4b}
We further evaluate \method\ with a Qwen3-4B student, using a Qwen3-14B teacher for mathematics and a Qwen3-8B teacher for code. Due to computational constraints, we compare against the base model and forward- and reverse-KL OPD baselines. On mathematics, \method\ achieves the strongest overall student performance, improving the base model from $27.1$ to $30.0$ avg@8 and from $45.5$ to $47.2$ pass@8 (Table~\ref{tab:math4b}). It also outperforms the best OPD baseline by $2.9$ avg@8 and $2.2$ pass@8 points, reaching within $0.4$ pass@8 points of the 14B teacher. On code, \method\ similarly achieves the highest overall student scores, reaching $57.3$ avg@8 and $70.5$ pass@8 compared with $54.1$ and $69.8$ for the base model (Table~\ref{tab:code4b}). Relative to the strongest OPD baseline, this corresponds to gains of $2.5$ avg@8 and $2.8$ pass@8 points. These results indicate that the benefits of combining outcome supervision with difficulty-gated teacher guidance extend to larger student models.

\begin{table}[t]
    \centering
    \small
    \renewcommand{\arraystretch}{1.12}
    \setlength{\tabcolsep}{3.2pt}
    \resizebox{\textwidth}{!}{%
    \begin{tabular}{lcccccccccc|cc}
        \toprule
        \rowcolor{headergray}
        \multirow{2}{*}{\textbf{Method}}
        & \multicolumn{2}{c}{\textbf{Minerva}}
        & \multicolumn{2}{c}{\textbf{Olympiad}}
        & \multicolumn{2}{c}{\textbf{AIME24}}
        & \multicolumn{2}{c}{\textbf{AIME25}}
        & \multicolumn{2}{c}{\textbf{HMMT25}}
        & \multicolumn{2}{c}{\textbf{Average}} \\
        \cmidrule(lr){2-3}
        \cmidrule(lr){4-5}
        \cmidrule(lr){6-7}
        \cmidrule(lr){8-9}
        \cmidrule(lr){10-11}
        \cmidrule(lr){12-13}
        \rowcolor{headergray}
        & \textit{avg@8} & \textit{pass@8}
        & \textit{avg@8} & \textit{pass@8}
        & \textit{avg@8} & \textit{pass@8}
        & \textit{avg@8} & \textit{pass@8}
        & \textit{avg@8} & \textit{pass@8}
        & \textit{avg@8} & \textit{pass@8} \\
        \midrule

        Base (Qwen3-4B)
        & 38.8 & 52.2
        & \underline{51.2} & \underline{71.7}
        & 20.4 & \underline{43.3}
        & \underline{17.5} & 35.6
        & 7.5 & \textbf{24.6}
        & \underline{27.1} & \underline{45.5} \\

        OPD (forward KL)
        & 38.9 & \textbf{56.3}
        & 48.5 & 70.4
        & \underline{20.8} & 36.5
        & 15.0 & 32.4
        & 7.9 & 19.1
        & 26.2 & 43.0 \\

        OPD (reverse KL)
        & \underline{39.0} & \underline{54.5}
        & 49.4 & 70.3
        & \underline{20.8} & 40.3
        & \underline{17.5} & \underline{39.2}
        & \underline{8.8} & 20.8
        & \underline{27.1} & 45.0 \\

        \rowcolor{methodblue}
        \textbf{\method}
        & \textbf{41.8} & 54.3
        & \textbf{53.3} & \textbf{72.1}
        & \textbf{22.1} & \textbf{47.1}
        & \textbf{21.2} & \textbf{40.6}
        & \textbf{11.7} & \underline{21.7}
        & \textbf{30.0} & \textbf{47.2} \\

        \midrule

        \rowcolor{teachergray}
        \textit{Teacher 14B}
        & 43.6 & 57.4
        & 55.4 & 74.8
        & 30.0 & 49.2
        & 20.0 & 34.5
        & 9.2 & 22.1
        & 31.6 & 47.6 \\

        \bottomrule
    \end{tabular}}
    \vspace{-0.2cm}
    \caption{\small
    \textbf{Mathematical reasoning with a Qwen3-4B student and Qwen3-14B teacher.}
    Avg@8 and pass@8 across five mathematics benchmarks.
    \textbf{Bold} and \underline{underline} denote the best and second-best student results, respectively.}
    \label{tab:math4b}
    \vspace{-0.2cm}
\end{table}
\begin{table}[t]
    \centering
    \small
    \renewcommand{\arraystretch}{1.12}
    \setlength{\tabcolsep}{3.5pt}

    \resizebox{\textwidth}{!}{%
    \begin{tabular}{lcccccc|cc}
        \toprule

        \rowcolor{headergray}
        \multirow{2}{*}{\textbf{Method}}
        & \multicolumn{2}{c}{\textbf{MBPP+}}
        & \multicolumn{2}{c}{\textbf{HumanEval+}}
        & \multicolumn{2}{c}{\textbf{LCB}}
        & \multicolumn{2}{c}{\textbf{Average}} \\

        \cmidrule(lr){2-3}
        \cmidrule(lr){4-5}
        \cmidrule(lr){6-7}
        \cmidrule(lr){8-9}
        \rowcolor{headergray}
        & \textit{avg@8} & \textit{pass@8}
        & \textit{avg@8} & \textit{pass@8}
        & \textit{avg@8} & \textit{pass@8}
        & \textit{avg@8} & \textit{pass@8} \\

        \midrule

        Base (Qwen3-4B)
        & 62.4 & \underline{78.5}
        & \underline{73.9} & \textbf{89.8}
        & \underline{26.0} & \underline{41.1}
        & 54.1 & \underline{69.8} \\

        OPD (forward KL)
        & 64.0 & \textbf{78.7}
        & 67.8 & 86.2
        & 22.8 & 38.1
        & 51.6 & 67.7 \\

        OPD (reverse KL)
        & \underline{66.9} & 77.6
        & 71.8 & 85.7
        & 25.6 & 39.7
        & \underline{54.8} & 67.7 \\

        \rowcolor{methodblue}
        \textbf{\method}
        & \textbf{67.1} & \textbf{78.7}
        & \textbf{77.1} & \underline{89.5}
        & \textbf{27.6} & \textbf{43.3}
        & \textbf{57.3} & \textbf{70.5} \\

        \midrule

        \rowcolor{teachergray}
        \textit{Teacher 8B}
        & 70.8 & 81.6
        & 80.0 & 90.2
        & 28.1 & 43.1
        & 59.6 & 71.7 \\

        \bottomrule
    \end{tabular}}
    \vspace{-0.2cm}
    \caption{\small
    \textbf{Code generation with a Qwen3-4B student and Qwen3-8B teacher.}
    Avg@8 and pass@8 on MBPP+, HumanEval+, and LiveCodeBench.
    \textbf{Bold} and \underline{underline} denote the best and second-best student results, respectively.}
    \label{tab:code4b}
    \vspace{-0.2cm}
\end{table}

\subsection{Rollout Group Size}
\label{app:group_size}
\method is motivated by a simple observation: group-relative rewards provide no learning signal when all sampled rollouts receive the same outcome. Smaller rollout groups make such reward-uniform groups more common, providing a direct test of this mechanism. We therefore retrained GRPO and \method with $n=4$ instead of $n=16$, keeping the data, teacher, optimizer, and training budget fixed (see Table~\ref{tab:n4}). Reducing $n$ from $16$ to $4$ lowers GRPO by $1.9$ avg@8 points on average across the four settings, compared with $1.2$ for \method. The difference is largest on 0.6B mathematics, where GRPO drops by $3.3$ points while \method drops by $1.6$. On code, \method also degrades less at both model sizes. We summarize the key points as follows:

\begin{itemize}
    \setlength{\itemsep}{1pt}
    \item \textbf{The advantage appears where reward is sparse.} On 0.6B mathematics \method\ moves from slightly below GRPO at $n=16$ to $+1.2$ avg@8 ahead at $n=4$, and it keeps the higher pass@8 in all four settings at both group sizes. Small groups produce more all-failure batches; there, the outcome signal is zero while \method\ still receives dense teacher guidance.
    \item \textbf{The gap is largest early in training.} With $n=4$ on 0.6B mathematics, \method\ reaches $11.8$ avg@8 by step $25$ against $7.3$ for GRPO. All-failure groups are most frequent at the start and become rarer as the student improves, and the gap narrows with them.
\end{itemize}

\begin{table}[t]
    \centering
    \small
    \renewcommand{\arraystretch}{1.13}
    \setlength{\tabcolsep}{4.2pt}

    \resizebox{\textwidth}{!}{%
    \begin{tabular}{lcccccccc}
        \toprule

        \rowcolor{headergray}
        \multirow{1}{*}{\textbf{Method}}
        & \multicolumn{4}{c}{\textbf{Mathematics}}
        & \multicolumn{4}{c}{\textbf{Code}} \\

        \cmidrule(lr){2-5}\cmidrule(lr){6-9}
        \rowcolor{headergray}
        & \multicolumn{2}{c}{\textbf{0.6B}}
        & \multicolumn{2}{c}{\textbf{1.7B}}
        & \multicolumn{2}{c}{\textbf{0.6B}}
        & \multicolumn{2}{c}{\textbf{1.7B}} \\

        \cmidrule(lr){2-3}\cmidrule(lr){4-5}
        \cmidrule(lr){6-7}\cmidrule(lr){8-9}
        \rowcolor{headergray}
        & \textit{avg@8} & \textit{pass@8}
        & \textit{avg@8} & \textit{pass@8}
        & \textit{avg@8} & \textit{pass@8}
        & \textit{avg@8} & \textit{pass@8} \\

        \midrule

        \rowcolor{sectiongray}
        \multicolumn{9}{l}{$n=16$ rollouts per prompt} \\

        \quad GRPO
        & \textbf{14.2} & 28.4
        & \textbf{21.4} & 35.2
        & 31.3 & 48.2
        & 43.4 & 55.7 \\

        \rowcolor{methodblue}
        \quad \textbf{\method}
        & 13.7 & \textbf{29.5}
        & 21.1 & \textbf{39.1}
        & \textbf{33.0} & \textbf{49.8}
        & \textbf{45.2} & \textbf{61.4} \\

        \midrule

        \rowcolor{sectiongray}
        \multicolumn{9}{l}{$n=4$ rollouts per prompt} \\

        \quad GRPO
        & 10.9 & 23.9
        & \textbf{21.3} & 36.0
        & 29.8 & 48.2
        & 40.8 & 52.3 \\

        \rowcolor{methodblue}
        \quad \textbf{\method}
        & \textbf{12.1} & \textbf{27.9}
        & 20.6 & \textbf{39.3}
        & \textbf{31.4} & \textbf{49.2}
        & \textbf{44.0} & \textbf{60.4} \\

        \bottomrule
    \end{tabular}}
    \vspace{-0.2cm}
    \caption{\small \textbf{Effect of rollout group size.} GRPO and \method\ trained with $n{=}16$ and $n{=}4$ rollouts per prompt, holding data, teacher, optimizer, and budget fixed; averages over the benchmark sets of Tables~\ref{tab:math} and~\ref{tab:code}. The $n{=}16$ rows are the runs of Tables~\ref{tab:math} and~\ref{tab:code}. Bold marks the better of the two arms within each group size.}
    \vspace{-0.2cm}
    \label{tab:n4}
\end{table}

\subsection{Bounded versus Unbounded Teacher Term} \label{app:bound}
The unbounded variant replaces $2\tanh(\Delta_{i,t}/2)$ with the raw log-probability gap $\Delta_{i,t}$ and changes nothing else. Bounding improves avg@8 by $1.0$--$1.5$ points and pass@8 by $0.8$--$1.3$ points in all four settings (see Table~\ref{tab:bound}); the gradient statistics are in the lower panel of the table and in Figure~\ref{fig:grad_norm_bars}.

\begin{table}[t]
    \centering
    \small
    \renewcommand{\arraystretch}{1.13}
    \setlength{\tabcolsep}{4.2pt}

    \resizebox{\textwidth}{!}{%
    \begin{tabular}{lcccccccc}
        \toprule

        \rowcolor{headergray}
        \textbf{Method}
        & \multicolumn{4}{c}{\textbf{Mathematics}}
        & \multicolumn{4}{c}{\textbf{Code}} \\

        \cmidrule(lr){2-5}\cmidrule(lr){6-9}
        \rowcolor{headergray}
        & \multicolumn{2}{c}{\textbf{0.6B}}
        & \multicolumn{2}{c}{\textbf{1.7B}}
        & \multicolumn{2}{c}{\textbf{0.6B}}
        & \multicolumn{2}{c}{\textbf{1.7B}} \\

        \cmidrule(lr){2-3}\cmidrule(lr){4-5}
        \cmidrule(lr){6-7}\cmidrule(lr){8-9}
        \rowcolor{headergray}
        & \textit{avg@8} & \textit{pass@8}
        & \textit{avg@8} & \textit{pass@8}
        & \textit{avg@8} & \textit{pass@8}
        & \textit{avg@8} & \textit{pass@8} \\

        \midrule

        \rowcolor{sectiongray}
        \multicolumn{9}{l}{\textbf{Held-out accuracy}} \\

        \rowcolor{methodblue}
        \textbf{\method, bounded $2\tanh(\Delta/2)$}
        & \textbf{13.7} & \textbf{29.5}
        & \textbf{21.1} & \textbf{39.1}
        & \textbf{33.0} & \textbf{49.8}
        & \textbf{45.2} & \textbf{61.4} \\

        \rowcolor{privblue}
        \method, unbounded (raw $\Delta$)
        & 12.7 & 28.2
        & 19.9 & 38.0
        & 31.9 & 49.0
        & 43.7 & 60.3 \\

        \rowcolor{sectiongray}
        \quad bounded $-$ unbounded
        & $+1.0$ & $+1.3$
        & $+1.2$ & $+1.1$
        & $+1.1$ & $+0.8$
        & $+1.5$ & $+1.1$ \\

        \midrule

        \rowcolor{sectiongray}
        \multicolumn{9}{l}{
        \textbf{Training gradient norm}
        \textit{(median / maximum over steps)}
        } \\

        \rowcolor{headergray}
        \textbf{Method}
        & \textit{median} & \textit{max}
        & \textit{median} & \textit{max}
        & \textit{median} & \textit{max}
        & \textit{median} & \textit{max} \\

        \rowcolor{methodblue}
        \textbf{\method, bounded}
        & 0.31 & 1.7
        & 0.21 & 0.7
        & 0.97 & 4.3
        & 0.80 & 2.6 \\

        \rowcolor{privblue}
        \method, unbounded
        & 0.50 & 3.4
        & 0.29 & 1.0
        & 3.22 & 342.3
        & 1.98 & 10.2 \\

        \bottomrule
    \end{tabular}}
    \vspace{-0.2cm}
    \caption{\small \textbf{Bounded versus unbounded teacher term.} Top: held-out avg@8 / pass@8 of the two \method\ variants. Bottom: median and maximum student gradient norm over training.}
    \vspace{-0.2cm}
    \label{tab:bound}
\end{table}

\textbf{Why the bound matters.} The teacher term $2\tanh(\Delta/2)$ passes small gaps through
unchanged and caps large ones at $2$. Most tokens are small gaps. A few are not: at the start of training, $2$--$9\%$ of the tokens the teacher scores have a gap above $2$ nats, and the largest in a step is $6$--$15$ nats (Table~\ref{tab:opd_side}). Unbounded, one such token enters the update with more weight than any GRPO advantage, which is at most $3.75$; bounded, with at most $2$. These tokens are most common on code, and that is where removing the bound does the most damage: the largest gradient step on 0.6B code grows from $4.3$ to $342$ (Figure~\ref{fig:grad_norm_bars}). Accuracy is also higher with the bound in every setting (Table~\ref{tab:bound}).

\begin{figure}[t]
    \centering
    \includegraphics[width=0.78\linewidth]{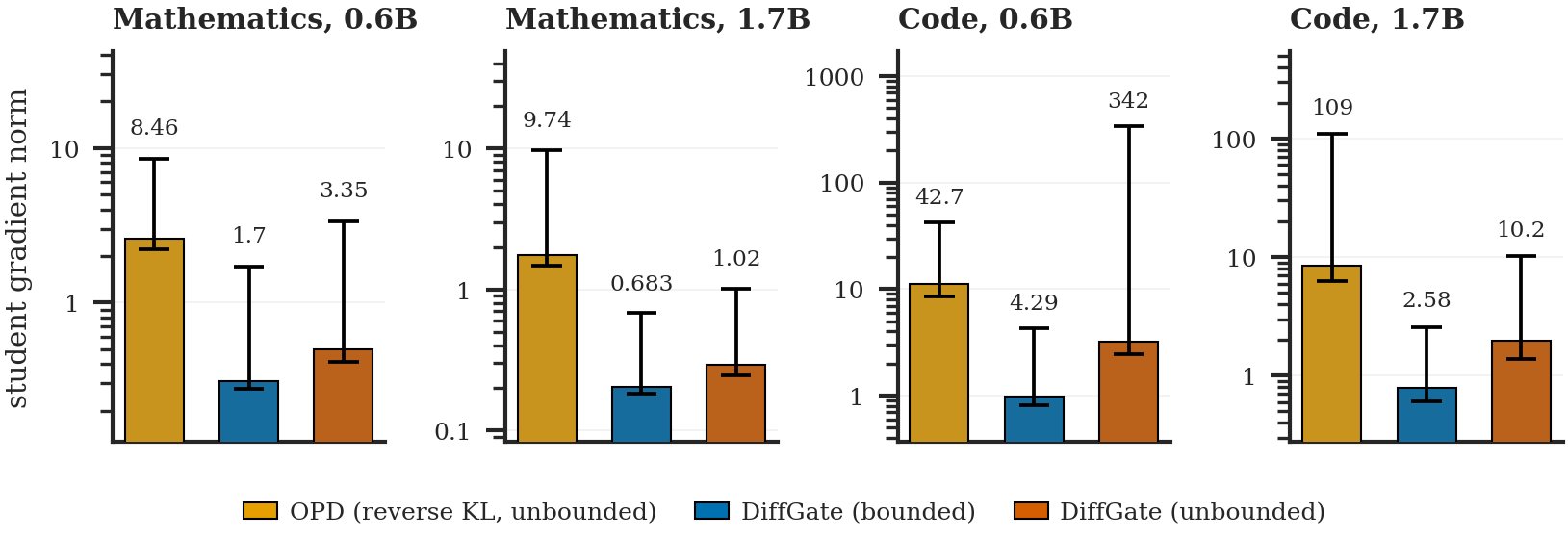}
    \vspace{-0.2cm}
    \caption{\small \textbf{Student gradient norm over training.} Bars: median. Whiskers: minimum to maximum step, with the maximum printed above. Log scale.}
    \vspace{-0.2cm}
    \label{fig:grad_norm_bars}
\end{figure}

\subsection{Difficulty Gating and the Additive GRPO$+$OPD Baseline}
\label{app:gating}
Table~\ref{tab:gating} compares \method\ with the additive objective, which applies the bounded teacher term with a constant weight to every failed trajectory. The gate matters most on mathematics, where it raises avg@8 by $2.4$ points at both model sizes and pass@8 by $2.8$ and $4.1$ points. On code, most of the effect of the teacher term is already present without the gate: gating adds $1.3$ and $1.4$ avg@8 and $0.1$ and $1.4$ pass@8. Appendix~\ref{app:one_view} shows the mechanism: the gate concentrates the teacher on groups where the reward is silent, whereas the additive rule spends it on groups the reward already handles.

\begin{table}[t]
    \centering
    \small
    \renewcommand{\arraystretch}{1.1}
    \setlength{\tabcolsep}{4.2pt}
    \resizebox{\textwidth}{!}{%
    \begin{tabular}{lcccccccc}
        \toprule
        \rowcolor{headergray}
        \multirow{1}{*}{\textbf{Objective}}
        & \multicolumn{4}{c}{\textbf{Mathematics}}
        & \multicolumn{4}{c}{\textbf{Code}} \\

        \cmidrule(lr){2-5}\cmidrule(lr){6-9}
        \rowcolor{headergray}
        & \multicolumn{2}{c}{\textbf{0.6B}} & \multicolumn{2}{c}{\textbf{1.7B}}
        & \multicolumn{2}{c}{\textbf{0.6B}} & \multicolumn{2}{c}{\textbf{1.7B}} \\

        \cmidrule(lr){2-3}\cmidrule(lr){4-5}\cmidrule(lr){6-7}\cmidrule(lr){8-9}
        \rowcolor{headergray}
        & \textit{avg@8} & \textit{pass@8} & \textit{avg@8} & \textit{pass@8}
        & \textit{avg@8} & \textit{pass@8} & \textit{avg@8} & \textit{pass@8} \\
        \midrule
        GRPO + OPD (additive, $\gamma{=}0$)
        & 11.3 & 26.7 & 18.7 & 35.0 & 31.7 & 49.7 & 43.8 & 60.0 \\

        \rowcolor{methodblue}
        \textbf{\method\ ($\gamma{=}1$)}
        & \textbf{13.7} & \textbf{29.5} & \textbf{21.1} & \textbf{39.1} & \textbf{33.0} & \textbf{49.8} & \textbf{45.2} & \textbf{61.4} \\

        \bottomrule
    \end{tabular}}
    \vspace{-0.2cm}
    \caption{\small \textbf{Difficulty gating against the additive baseline.} Both arms are GRPO plus the same bounded teacher term on failed trajectories, matched on data, teacher, $n{=}16$ rollouts, and step budget. They differ only in that term's weight: constant (additive, $\gamma{=}0$) against the failed fraction of the group (\method, $\gamma{=}1$).}
    \vspace{-0.2cm}
    \label{tab:gating}
\end{table}

\subsection{Selector and Additive Controls}
\label{app:selector}

The teacher contribution in \method\ combines three design choices beyond its addition to GRPO: a bounded token-level signal, a failed-only selector, and group-difficulty weighting. Table~\ref{tab:selector} separates these components using three additional controls. The \emph{all-rollout additive} objective uses the bounded teacher signal with constant weight, $w_i=1$. The \emph{all-rollout raw} objective additionally replaces the bounded signal with the raw gap $\Delta_{i,t}$, yielding a KDRL-style GRPO$+$reverse-KL objective at our teacher weight $\lambda=1$. Finally, \method\ without the failed-only selector retains difficulty weighting but uses $w_i=(1-p)^\gamma$ for every rollout. The full \method\ instead uses $w_i=(1-p)^\gamma(1-R_i)$. All controls use the same data, teacher, $n=16$ rollouts, and training-step budget. Each cell corresponds to a single run; differences smaller than the observed run-to-run variation are therefore interpreted cautiously.

\textbf{All-rollout teacher guidance increases coverage without consistent avg@8 gains.}
Relative to GRPO (Tables~\ref{tab:math} and~\ref{tab:code}), the bounded all-rollout additive objective lowers mathematics avg@8 by $2.3$ and $1.2$ points and changes code avg@8 by at most $0.2$ points. Pass@8 is higher in three of the four settings, including a $5.6$-point increase on 1.7B code. Thus, in these controls, uniform teacher guidance tends to increase solution coverage without consistently improving per-sample accuracy.

\textbf{Bounding does not consistently improve the all-rollout objective.}
Relative to the bounded all-rollout objective, using the raw gap increases mathematics avg@8 by $1.6$ and $0.8$ points. On code, avg@8 differs by at most $1.2$ points, while the raw gap lowers pass@8 by $1.7$ and $2.3$ points. Thus, the benefit of bounding observed in Table~\ref{tab:bound}, where teacher guidance is restricted to failed rollouts, does not transfer uniformly to all-rollout supervision.

\textbf{The failed-only selector has its clearest effect on 1.7B mathematics.}
Removing the selector lowers avg@8 by $1.3$ points and pass@8 by $2.6$ points on 1.7B mathematics. In the other three settings, avg@8 changes by at most $0.6$ points, indicating that the contribution of the selector varies across model size and domain.

\textbf{Difficulty weighting contributes strongly to the gains over uniform additive guidance.}
When teacher guidance is applied to every rollout, replacing constant weighting with $(1-p)^\gamma$ raises code avg@8 by $1.0$ and $2.6$ points for the 0.6B and 1.7B students, respectively. It also improves 0.6B mathematics by $1.4$ points, while changing 1.7B mathematics by $-0.4$ points. These results indicate that adapting teacher strength to group difficulty is an important component of the improvement over uniform additive supervision, particularly on code.

\textbf{\method\ improves over the KDRL-style all-rollout objective most clearly on code.}
Relative to the raw all-rollout GRPO$+$reverse-KL control, \method\ improves code avg@8 by $1.4$ and $3.2$ points and pass@8 by $2.0$ and $2.4$ points for the 0.6B and 1.7B students, respectively. On mathematics, the two objectives differ by at most $0.2$ avg@8 and $1.0$ pass@8 points. The clearest benefit of selective, difficulty-weighted teacher guidance therefore appears on code, where reward-uniform groups also persist longer during training (Appendix~\ref{app:adv_dynamics}).

\begin{table}[t]
    \centering
    \small
    \renewcommand{\arraystretch}{1.1}
    \setlength{\tabcolsep}{4.2pt}
    \resizebox{\textwidth}{!}{%
    \begin{tabular}{lcccccccc}
        \toprule
        \rowcolor{headergray}
        \multirow{1}{*}{\textbf{Objective}}
        & \multicolumn{4}{c}{\textbf{Mathematics}}
        & \multicolumn{4}{c}{\textbf{Code}} \\

        \cmidrule(lr){2-5}\cmidrule(lr){6-9}
        \rowcolor{headergray}
        & \multicolumn{2}{c}{\textbf{0.6B}} & \multicolumn{2}{c}{\textbf{1.7B}}
        & \multicolumn{2}{c}{\textbf{0.6B}} & \multicolumn{2}{c}{\textbf{1.7B}} \\

        \cmidrule(lr){2-3}\cmidrule(lr){4-5}\cmidrule(lr){6-7}\cmidrule(lr){8-9}
        \rowcolor{headergray}
        & \textit{avg@8} & \textit{pass@8} & \textit{avg@8} & \textit{pass@8}
        & \textit{avg@8} & \textit{pass@8} & \textit{avg@8} & \textit{pass@8} \\
        \midrule
        GRPO + bounded OPD (all rollouts)
        & 11.9 & 27.4 & 20.2 & 36.8 & 31.4 & 49.5 & 43.2 & 61.3 \\
        GRPO + OPD (additive, all rollouts, raw $\Delta$)
        & \underline{13.5} & \textbf{29.9} & \underline{21.0} & \underline{38.1} & 31.6 & 47.8 & 42.0 & 59.0 \\
        GRPO + bounded OPD (difficulty-weighted, all rollouts)
        & 13.3 & 28.4 & 19.8 & 36.5 & \underline{32.4} & \underline{49.7} & \textbf{45.8} & \textbf{61.7} \\
        \rowcolor{methodblue}
        \textbf{\method}
        & \textbf{13.7} & \underline{29.5} & \textbf{21.1} & \textbf{39.1} & \textbf{33.0} & \textbf{49.8} & \underline{45.2} & \underline{61.4} \\
        \bottomrule
    \end{tabular}}
    \vspace{-0.2cm}
    \caption{\small \textbf{Selector and additive controls.} All rows add the teacher term to GRPO; the first three apply it to every rollout. \textbf{Bold} and \underline{underlined}: best and second-best per column.}
    \vspace{-0.2cm}
    \label{tab:selector}
\end{table}

\subsection{Advantage Dynamics: \method\ versus OPD and GRPO} \label{app:adv_dynamics}
When all sixteen rollouts for a prompt receive the same reward, GRPO provides no learning signal because every trajectory has zero group-relative advantage. In an all-fail group, however, \method\ can still use the teacher to provide token-level guidance; in an all-pass group, the teacher gate is intentionally zero. We use the training logs to examine how often such reward-uniform groups occur, whether restricting teacher guidance to failed rollouts leaves substantial teacher--student disagreement, and whether the gate concentrates supervision where GRPO is least informative. GRPO and \method\ use the same data, teacher, group size, and training budget, and their statistics match at step~1, so subsequent differences reflect the training objective.

\textbf{Reward-uniform groups are common, especially on code.}
At step~1, all rollouts fail for $32$--$36\%$ of mathematics prompts and $14$--$35\%$ of code prompts. This fraction decreases rapidly on mathematics as the student improves, but remains substantial on code: roughly $20\%$ of 0.6B code prompts are still all-fail near the end of training. Including all-pass groups, $59$--$74\%$ of code prompts can provide no GRPO signal at later checkpoints (Table~\ref{tab:adv_dynamics}). \method\ addresses precisely the all-fail case by applying teacher guidance when GRPO is silent, while leaving already-solved all-pass groups unchanged. It directly explains why teacher guidance can be especially useful on code, where reward-uniform groups remain frequent throughout training.

\begin{figure}[t]
    \centering
    \includegraphics[width=\linewidth]{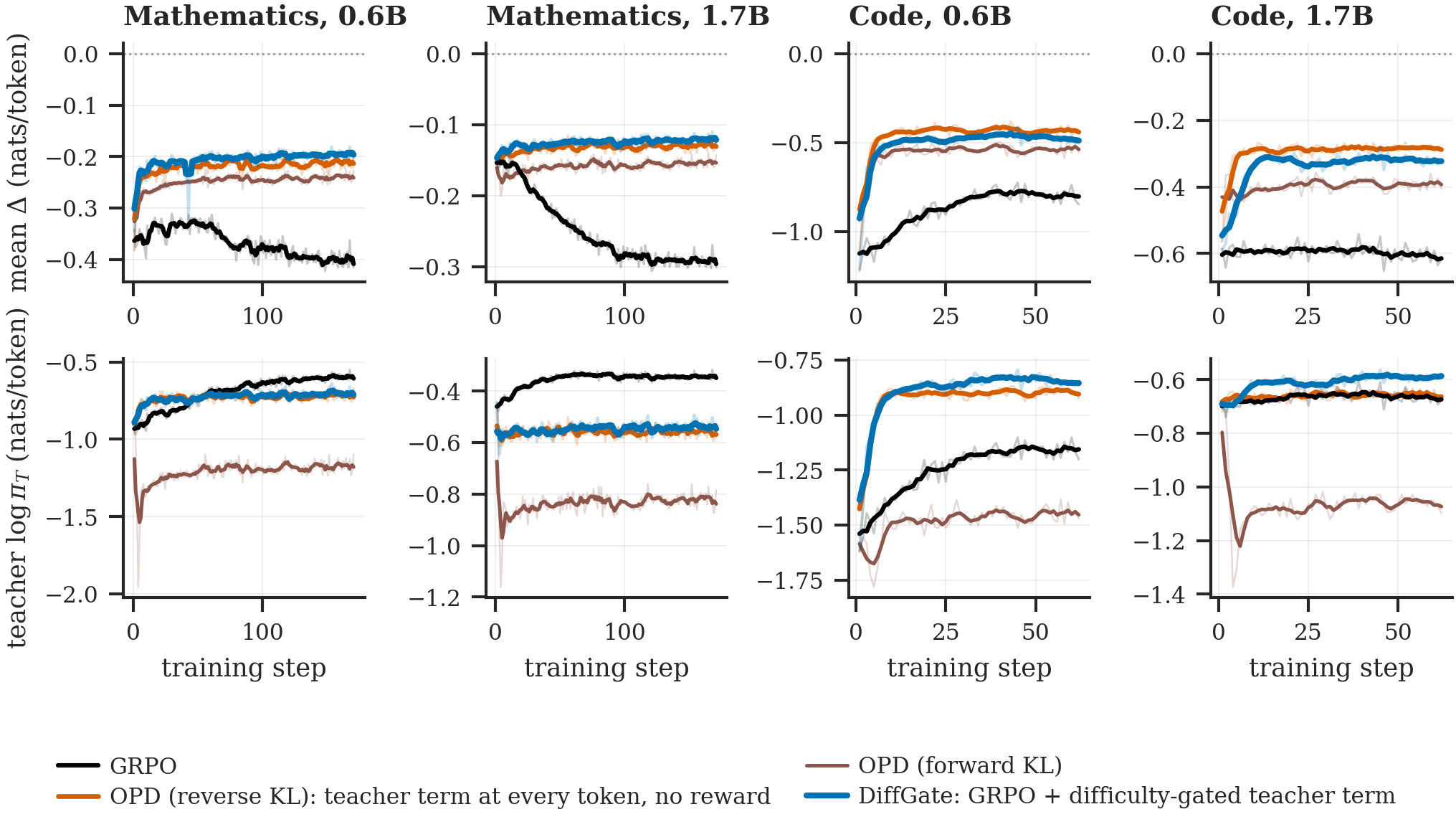}
    \vspace{-0.2cm}
    \caption{\small \textbf{Teacher-side training dynamics: GRPO, OPD and \method.} Columns are the four model--domain settings. GRPO and \method\ use $n{=}16$ rollouts and the outcome reward, \method\ adding the gated teacher term; the two OPD baselines use one rollout per prompt and no reward. All four share data, teacher, and step budget. Faint traces are per-step values, bold lines are a five-step rolling mean. Rows: mean $\Delta$ over all response tokens, and the teacher's log-likelihood of the student's own samples.}
    \vspace{-0.2cm}
    \label{fig:adv_dynamics}
\end{figure}

\begin{table}[t]
    \centering
    \small
    \renewcommand{\arraystretch}{1.13}
    \setlength{\tabcolsep}{3.2pt}

    \resizebox{\textwidth}{!}{%
    \begin{tabular}{lcccccccccccc}
        \toprule

        \rowcolor{headergray}
        \textbf{Quantity}
        & \multicolumn{3}{c}{\textbf{Mathematics, 0.6B}}
        & \multicolumn{3}{c}{\textbf{Mathematics, 1.7B}}
        & \multicolumn{3}{c}{\textbf{Code, 0.6B}}
        & \multicolumn{3}{c}{\textbf{Code, 1.7B}} \\

        \cmidrule(lr){2-4}\cmidrule(lr){5-7}
        \cmidrule(lr){8-10}\cmidrule(lr){11-13}
        \rowcolor{headergray}
        &
        \textit{Start} & \textit{GRPO} & \textit{\method}
        & \textit{Start} & \textit{GRPO} & \textit{\method}
        & \textit{Start} & \textit{GRPO} & \textit{\method}
        & \textit{Start} & \textit{GRPO} & \textit{\method} \\

        \midrule

        \rowcolor{sectiongray}
        \multicolumn{13}{l}{\textbf{Group-level reward statistics}} \\

        All-fail groups, $A_{\mathrm{GRPO}}\equiv0$ (\%)
        & 32.4 & 7.5 & 10.8
        & 35.6 & 2.6 & 2.6
        & 35.4 & 22.3 & 19.6
        & 13.7 & 9.2 & 5.7 \\

        All-pass groups, $A_{\mathrm{GRPO}}\equiv0$ (\%)
        & 22.6 & 60.9 & 49.1
        & 32.1 & 70.8 & 62.1
        & 24.0 & 36.5 & 37.4
        & 49.8 & 55.8 & 56.3 \\

        Prompts with no group signal (\%)
        & 55.0 & 68.5 & 59.8
        & 67.7 & 73.5 & 64.8
        & 59.4 & 58.8 & 57.0
        & 63.5 & 65.0 & 62.1 \\

        Token-weighted mean $A_{\mathrm{GRPO}}$
        & -0.010 & -0.051 & -0.030
        & -0.010 & -0.052 & -0.037
        & +0.004 & -0.065 & -0.023
        & -0.014 & -0.053 & -0.023 \\

        \midrule

        \rowcolor{sectiongray}
        \multicolumn{13}{l}{\textbf{Teacher signal on failed rollouts}} \\

        $|\Delta|$ on failed rollouts (nats/token)
        & 0.386 & 0.170 & 0.146
        & 0.166 & 0.087 & 0.078
        & 1.039 & 0.556 & 0.374
        & 0.315 & 0.295 & 0.188 \\

        Bounded $|2\tanh(\Delta/2)|$ on failed rollouts
        & 0.245 & 0.117 & 0.111
        & 0.129 & 0.065 & 0.063
        & 0.310 & 0.242 & 0.199
        & 0.115 & 0.112 & 0.107 \\

        Difficulty gate $(1-p)$ on failed rollouts
        & 0.69 & 0.33 & 0.41
        & 0.63 & 0.21 & 0.26
        & 0.69 & 0.55 & 0.54
        & 0.43 & 0.36 & 0.33 \\

        Failed-rollout tokens with $|\Delta|>2$ (\%)
        & 5.6 & 2.3 & 1.6
        & 1.7 & 1.0 & 0.7
        & 11.2 & 6.8 & 4.9
        & 3.7 & 3.6 & 2.4 \\

        \midrule

        \rowcolor{sectiongray}
        \multicolumn{13}{l}{\textbf{Policy and training statistics}} \\

        Mean $\Delta$, all tokens (nats/token)
        & -0.343 & -0.401 & -0.195
        & -0.143 & -0.293 & -0.121
        & -1.213 & -0.802 & -0.479
        & -0.587 & -0.609 & -0.321 \\

        Mean $\log \pi_T$, all tokens (nats/token)
        & -0.886 & -0.599 & -0.708
        & -0.445 & -0.344 & -0.542
        & -1.609 & -1.164 & -0.851
        & -0.678 & -0.669 & -0.592 \\

        Train solve rate (\%)
        & 22.4 & 60.8 & 49.3
        & 30.8 & 71.2 & 63.4
        & 22.6 & 38.0 & 37.1
        & 49.4 & 55.6 & 56.4 \\

        \bottomrule
    \end{tabular}}
    \vspace{-0.2cm}
    \caption{\small \textbf{Advantage and teacher-signal statistics over training.} ``Start'' is the mean of the two arms at step~1; GRPO and \method\ report averages over the final ten steps. Reward-uniform groups (all-fail or all-pass) have zero group-relative advantage. We measure teacher-signal statistics only on failed rollouts, and we average over the final block of all generated response tokens.}
    \vspace{-0.2cm}
    \label{tab:adv_dynamics}
\end{table}

\textbf{Selective teacher guidance remains effective.} Figure~\ref{fig:adv_dynamics} tracks the teacher--student gap and the teacher log-likelihood of student-generated samples throughout training. By the end of training, \method\ reaches a teacher--student gap comparable to reverse-KL OPD, despite applying teacher guidance only to failed rollouts. In contrast, the gap under GRPO increases on mathematics, from $0.15$ to $0.29$ for the 1.7B student. At the same time, the teacher assigns increasing likelihood to GRPO samples, suggesting that the larger gap does not simply reflect poorer generations. Together with the entropy trends in Figure~\ref{fig:why_diffgate}b, this pattern is consistent with GRPO becoming more concentrated on a narrower subset of teacher-supported responses. GRPO remains competitive on avg@8 while \method\ achieves higher pass@8.

\textbf{The gate focuses teacher guidance on failed and difficult rollouts.} Figure~\ref{fig:opd_side} examines the quantities directly used by the gate. Teacher disagreement is initially similar on correct and failed rollouts but becomes larger on correct rollouts later in training. By the end, it is $1.3$--$1.7\times$ higher on correct rollouts under \method\ and $1.5$--$3.5\times$ under GRPO (Table~\ref{tab:opd_side}). Applying teacher guidance uniformly would therefore spend substantial supervision on trajectories that already satisfy the verifier. The factor $(1-R_i)$ avoids this by restricting teacher guidance to failed rollouts. The average applied teacher weight also decreases as performance improves, from roughly $0.5$ to $0.1$--$0.2$ on mathematics, consistent with the intended difficulty-dependent behavior described in Section~\ref{sec:theory_relative_strength}.

\begin{figure}[t]
    \centering
    \includegraphics[width=\linewidth]{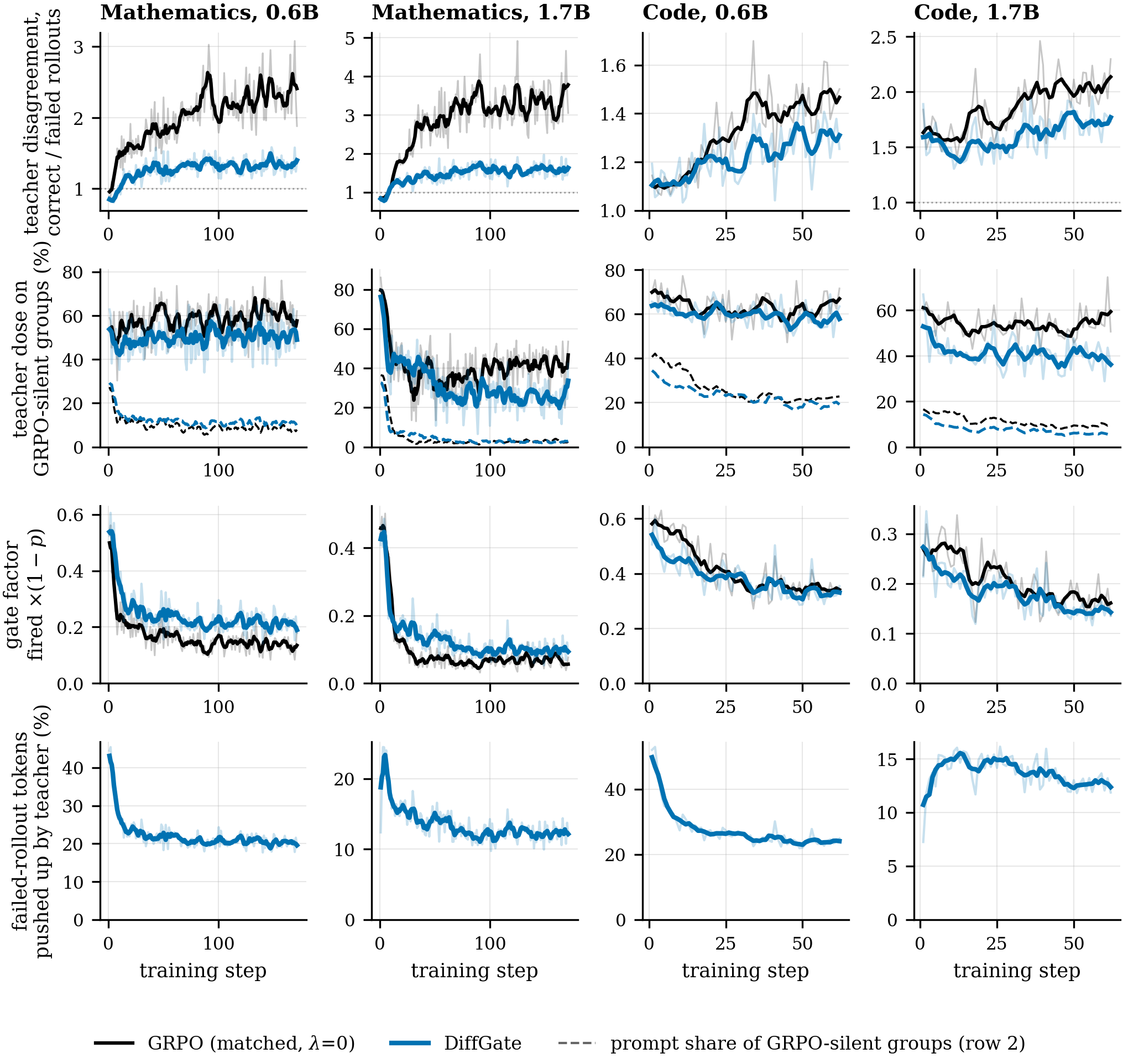}
    \vspace{-0.2cm}
    \caption{\small \textbf{Reward-based routing of teacher guidance.} the rows show: (1) the ratio of the all-token mean of $-\Delta$ to the failed-rollout mean of $|\Delta|$, with the dotted line marking a ratio of $1$; (2) the fraction of the applied teacher contribution assigned to all-failure groups, together with the fraction of prompts belonging to such groups (dashed); (3) the effective gate factor, defined by the fraction of trajectories on which the gate fires multiplied by their mean difficulty weight; and (4) for \method, the fraction of failed-rollout tokens on which the teacher signal opposes the outcome advantage.}
    \vspace{-0.2cm}
    \label{fig:opd_side}
\end{figure}

\begin{table}[t]
    \centering
    \small
    \renewcommand{\arraystretch}{1.13}
    \setlength{\tabcolsep}{3.2pt}

    \resizebox{\textwidth}{!}{%
    \begin{tabular}{lcccccccc}
        \toprule

        \rowcolor{headergray}
        \textbf{Quantity}
        & \multicolumn{2}{c}{\textbf{Mathematics, 0.6B}}
        & \multicolumn{2}{c}{\textbf{Mathematics, 1.7B}}
        & \multicolumn{2}{c}{\textbf{Code, 0.6B}}
        & \multicolumn{2}{c}{\textbf{Code, 1.7B}} \\

        \cmidrule(lr){2-3}\cmidrule(lr){4-5}
        \cmidrule(lr){6-7}\cmidrule(lr){8-9}
        \rowcolor{headergray}
        & \textit{GRPO} & \textit{\method}
        & \textit{GRPO} & \textit{\method}
        & \textit{GRPO} & \textit{\method}
        & \textit{GRPO} & \textit{\method} \\

        \midrule

        \rowcolor{sectiongray}
        \multicolumn{9}{l}{\textbf{Teacher disagreement and reward-signal availability}} \\

        Correct/failed disagreement ratio
        & $0.99{\to}2.41$ & $0.83{\to}1.34$
        & $0.85{\to}3.48$ & $0.80{\to}1.57$
        & $1.10{\to}1.45$ & $1.13{\to}1.29$
        & $1.68{\to}2.07$ & $1.62{\to}1.72$ \\

        Teacher dose on GRPO-silent groups (\%)
        & $55{\to}58$ & $52{\to}52$
        & $78{\to}41$ & $69{\to}26$
        & $69{\to}64$ & $63{\to}58$
        & $58{\to}56$ & $52{\to}39$ \\

        GRPO-silent groups, share of prompts (\%)
        & $26{\to}8$ & $28{\to}11$
        & $36{\to}3$ & $31{\to}3$
        & $40{\to}22$ & $32{\to}20$
        & $15{\to}9$ & $13{\to}6$ \\

        \midrule

        \rowcolor{sectiongray}
        \multicolumn{9}{l}{\textbf{Applied teacher guidance}} \\

        Gate factor, fired $\times(1-p)$
        & $0.48{\to}0.13$ & $0.54{\to}0.21$
        & $0.47{\to}0.06$ & $0.44{\to}0.10$
        & $0.58{\to}0.35$ & $0.50{\to}0.34$
        & $0.25{\to}0.16$ & $0.25{\to}0.14$ \\

        Applied teacher dose per token
        & $0.112{\to}0.016$ & $0.122{\to}0.023$
        & $0.065{\to}0.004$ & $0.061{\to}0.006$
        & $0.182{\to}0.086$ & $0.145{\to}0.068$
        & $0.033{\to}0.018$ & $0.034{\to}0.016$ \\

        Failed tokens pushed up (\%)
        & -- & $41{\to}20$
        & -- & $21{\to}12$
        & -- & $44{\to}24$
        & -- & $12{\to}13$ \\

        \midrule

        \rowcolor{sectiongray}
        \multicolumn{9}{l}{\textbf{Teacher confidence on failed-rollout tokens}} \\

        Teacher prob.\ on pushed-up tokens
        & -- & $0.68{\to}0.43$
        & -- & $0.56{\to}0.31$
        & -- & $0.69{\to}0.55$
        & -- & $0.49{\to}0.38$ \\

        Teacher prob.\ on all failed tokens
        & $0.54{\to}0.30$ & $0.56{\to}0.38$
        & $0.57{\to}0.25$ & $0.54{\to}0.30$
        & $0.61{\to}0.45$ & $0.55{\to}0.47$
        & $0.46{\to}0.37$ & $0.45{\to}0.35$ \\

        \midrule

        \rowcolor{sectiongray}
        \multicolumn{9}{l}{\textbf{Tail of the raw gap $\Delta$, before the bound}} \\

        Tokens with $|\Delta|>2$ nats (\%)
        & $5.2{\to}2.3$ & $4.1{\to}1.6$
        & $1.8{\to}1.0$ & $1.7{\to}0.7$
        & $11.1{\to}6.8$ & $8.7{\to}4.8$
        & $4.3{\to}3.6$ & $3.8{\to}2.4$ \\

        Largest $|\Delta|$ in a step (nats)
        & $8.9{\to}4.9$ & $9.3{\to}5.5$
        & $6.3{\to}3.3$ & $6.3{\to}3.1$
        & $17.7{\to}13.6$ & $14.6{\to}10.6$
        & $10.2{\to}8.5$ & $9.6{\to}6.3$ \\

        \midrule

        \rowcolor{sectiongray}
        \multicolumn{9}{l}{\textbf{Optimisation} \textit{(median / maximum over training)}} \\

        Student gradient norm
        & $0.22$ / $0.7$ & $0.31$ / $1.7$
        & $0.31$ / $212.9$ & $0.21$ / $0.7$
        & $0.67$ / $1.2$ & $0.97$ / $4.3$
        & $1.36$ / $64.9$ & $0.80$ / $2.6$ \\

        \bottomrule
    \end{tabular}}
    \vspace{-0.2cm}
    \caption{\small \textbf{Teacher-side diagnostics over training.} Each entry reports start $\rightarrow$ end, where start'' averages the first five training steps and end'' averages the final ten. For GRPO, teacher-derived quantities are computed only for diagnostic purposes and are never used for optimization; rows that require an applied teacher update are therefore left blank. The disagreement block reports the raw teacher--student log-probability difference $\Delta$ before bounding, evaluated on teacher-scored tokens. The gradient row instead reports the median/maximum gradient norm over the full training run.}
    \vspace{-0.2cm}
    \label{tab:opd_side}
\end{table}

\subsection{Training Reward and Response Length} \label{app:reward_dynamics}
Figure~\ref{fig:reward_dynamics} tracks the verifier pass rate over the $16$ rollouts, the mean response length, and the policy entropy for GRPO, reverse-KL OPD, \method\ and the additive variant with no difficulty weighting ($\gamma=0$). All arms share the training set and step budget.

\begin{figure}[t]
    \centering
    \includegraphics[width=\linewidth]{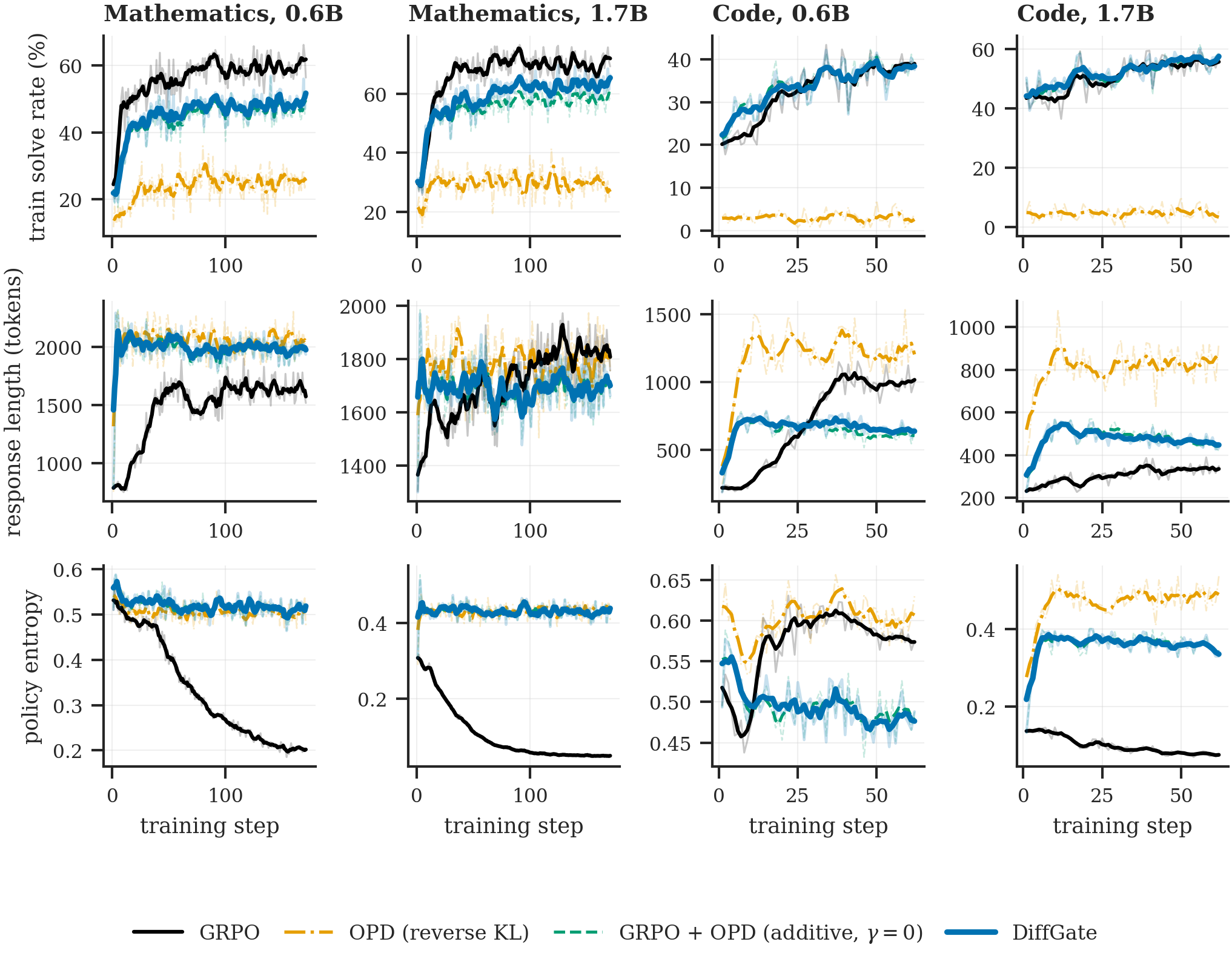}
    \vspace{-0.2cm}
    \caption{\small \textbf{Training reward, response length and policy entropy.} Rows: verifier pass rate of the $n{=}16$ rollouts on each step's training prompts; mean response length in tokens; mean token entropy of the student. Faint traces are per-step values, bold lines a five-step rolling mean. \method\ is a dense retrain; the remaining arms are the original runs. The reverse-KL arm's solve rate is a single-sample diagnostic, not the
    quantity it optimises.}
    \vspace{-0.2cm}
    \label{fig:reward_dynamics}
\end{figure}

\textbf{Teacher guidance affects response behavior early in training.} One of the clearest early changes is response length. Teacher-guided models show a rapid shift in response length during the first few training steps, coinciding with the decrease in teacher--student disagreement shown in Figure~\ref{fig:adv_dynamics}, while GRPO changes more gradually. After this initial adjustment, response lengths become relatively stable, consistent with later teacher guidance reflecting less of the early response-format differences.

\textbf{We observe no clear evidence of verifier exploitation.} Mathematics rewards use symbolic answer checks, while code rewards rely on hidden unit tests executed in a sandbox, avoiding reliance on a learned reward model. We do observe behaviors such as response-length growth and format specialization, particularly under GRPO, but these do not by themselves indicate reward exploitation. The improvements on code also transfer consistently to held-out evaluation benchmarks. Since we did not perform a dedicated audit for verifier exploits, we interpret this as an absence of observed evidence rather than evidence that exploitation cannot occur.

\subsection{Why the Gate, in One View}\label{app:one_view}
Figure~\ref{fig:why_diffgate} compares four aspects of training for the 1.7B mathematics setting. Together, they show that \method\ maintains stable updates, preserves higher policy entropy, gradually reduces the relative contribution of teacher guidance, and suppresses unnecessary teacher supervision as the student improves. Results for the remaining settings are provided in Figure~\ref{fig:grad_norm_bars} and Table~\ref{tab:opd_side}.

\textbf{Bounded teacher guidance is associated with more stable updates.}
\method\ has a lower median gradient norm than GRPO ($0.205$ versus $0.311$) and a much smaller maximum ($0.68$ versus $212.9$; Figure~\ref{fig:why_diffgate}a). The large GRPO gradients occur late in training, around the same period in which its policy entropy falls below $0.1$. Reverse-KL OPD instead maintains substantially larger gradients throughout training, with a median of $1.74$. Bounding the teacher signal also matters: replacing $2\tanh(\Delta/2)$ with the raw log-probability gap increases the median gradient norm from $0.205$ to $0.293$ and reduces evaluation performance from $21.1/39.1$ to $19.9/38.0$ avg@8/pass@8. Bounding thus limits extreme teacher contributions while retaining useful supervision.

\textbf{\method\ maintains a less concentrated policy.}
Figure~\ref{fig:why_diffgate}b shows the difference in entropy dynamics. Under GRPO, entropy decreases from approximately $0.30$ to $0.048$, crossing $0.1$ midway through training. In contrast, \method\ and reverse-KL OPD maintain substantially higher entropy, ending near $0.42$--$0.44$. The small difference between these two teacher-guided methods is within their normal step-to-step variation, so we interpret them as exhibiting similar entropy behavior. Together with the higher pass@8 of \method, this pattern is consistent with retaining greater diversity across sampled solutions, although it does not establish a causal relationship.

\textbf{Teacher guidance becomes relatively less important as training progresses.}
Figure~\ref{fig:why_diffgate}c compares the teacher and GRPO components of \method\ on failed rollouts. The applied teacher contribution decreases from $0.087$ per token during the first five steps to $0.017$ during the final ten steps. At the same time, all-failure groups become much less frequent, falling from roughly $31\%$ to $3\%$. Thus, early in training the teacher provides useful supervision on prompts where group-relative reward can be uninformative, while its contribution decreases as informative reward variation becomes more common. This behavior is consistent with the intended role of difficulty gating rather than requiring an explicit annealing schedule.

\textbf{Difficulty gating reduces unnecessary teacher supervision.}
Figure~\ref{fig:why_diffgate}d compares the teacher contribution on the same \method\ rollouts with and without difficulty weighting. Under the proposed gate, the applied teacher contribution per response token decreases from $0.061$ to $0.006$ over training, whereas the corresponding additive weighting decreases from $0.097$ to $0.024$. Consequently, the difference between the two weighting schemes becomes larger as the student improves. This follows directly from the factor $(1-p)^\gamma$: teacher guidance remains strongest on difficult and all-failure groups and is progressively reduced on groups the student solves more often. The corresponding ablation in Table~\ref{tab:gating} shows that this difficulty-dependent weighting improves performance over the additive variant in the 1.7B mathematics setting.

\subsection{Where Does \method\ Help? Difficulty and Teacher Reliability}
\label{app:bands}
We analyze how \method's gains vary with problem difficulty and teacher reliability. Held-out code problems are grouped by the number of correct responses produced by the base student out of $16$, ranging from never solved ($0/16$) to frequently solved ($12$--$16/16$). All methods use the same $16$-sample evaluation protocol (Table~\ref{tab:bands}), and within each band we compare \method\ and GRPO on the same problems. We additionally report exploratory sign tests; two of the eight band-level comparisons have unadjusted $p<0.05$. Over the full evaluation sets, \method\ also improves over reverse-KL OPD by $1.7$--$3.1$ avg@8 points and $1.3$--$2.8$ pass@8 points across the four model--domain settings (Tables~\ref{tab:math} and~\ref{tab:code}).

\begin{table}[t]
    \centering
    \small
    \renewcommand{\arraystretch}{1.13}
    \setlength{\tabcolsep}{3.4pt}

    \resizebox{\textwidth}{!}{%
    \begin{tabular}{lrrrrrrr}
        \toprule

        \rowcolor{headergray}
        \textbf{Base-student band}
        & \textbf{$n$}
        & \textbf{Teacher pass@16}
        & \textbf{Base}
        & \textbf{GRPO}
        & \textbf{\method}
        & \textbf{$\Delta$}
        & \textbf{Wins (\method:GRPO)} \\

        \midrule

        \rowcolor{sectiongray}
        \multicolumn{8}{l}{
        \textbf{Code, 0.6B student}
        \textit{(994 problems; teacher fails on 50\% of the never-solved band, 31\% overall)}
        } \\

        \quad Never solved (0/16)
        & 616 & 50\% & 0.00 & 0.86 & 0.79 & $-0.07$ & 16:21 \\

        \quad 1--4 / 16
        & 140 & 96\% & 2.13 & 6.46 & 6.64 & $+0.18$ & 22:17 \\

        \quad 5--11 / 16
        & 132 & 98\% & 8.08 & 11.22 & 12.35 & $+1.13$ & 33:17$^{\dagger}$ \\

        \quad 12--16 / 16
        & 106 & 98\% & 14.09 & 14.74 & 15.28 & $+0.55$ & 9:3 \\

        \midrule

        \rowcolor{sectiongray}
        \multicolumn{8}{l}{
        \textbf{Code, 1.7B student}
        \textit{(994 problems; teacher fails on 58\% of the never-solved band, 30\% overall)}
        } \\

        \quad Never solved (0/16)
        & 515 & 42\% & 0.00 & 0.90 & 1.42 & $+0.52$ & 45:16$^{\dagger}$ \\

        \quad 1--4 / 16
        & 109 & 94\% & 2.10 & 5.93 & 5.81 & $-0.12$ & 18:16 \\

        \quad 5--11 / 16
        & 67 & 96\% & 8.15 & 9.72 & 9.52 & $-0.19$ & 10:14 \\

        \quad 12--16 / 16
        & 303 & 98\% & 15.48 & 14.59 & 13.96 & $-0.63$ & 17:38$^{\dagger}$ \\

        \bottomrule
    \end{tabular}}
    \vspace{-0.2cm}
    \caption{\small \textbf{Performance across base-student difficulty bands.} Problems are grouped by the number of successful responses produced by the base student out of $16$. ``Teacher pass@16'' denotes the fraction of problems in each band for which the Qwen3-8B teacher produces at least one correct response under the same sampling protocol. Base, GRPO, and \method\ report the mean number of successful generations out of $16$, and $\Delta$ denotes \method\ minus GRPO. ``Wins'' counts problems favoring \method\ versus GRPO for which a two-proportion $z$-test comparing their $16$ sampled outcomes satisfies $|z|\geq2$, with \method\ listed first.}
    \vspace{-0.2cm}
    \label{tab:bands}
\end{table}

\textbf{Difficulty gating targets regions where additional guidance can be most useful.} Teacher reliability is lowest on the hardest code problems: teacher pass@16 is only $42$--$50\%$ in the never-solved band, compared with $94$--$98\%$ in the remaining bands. This motivates applying teacher guidance adaptively rather than uniformly. \method\ uses the student's group solve rate to emphasize harder rollout groups while progressively reducing teacher weight as the student becomes more successful.

\textbf{The largest gains occur in harder or intermediate difficulty regions.} For the 0.6B code student, \method\ achieves its largest improvement in the $5$--$11/16$ band, with $+1.13$ additional successful samples per problem. For the 1.7B student, the largest gain occurs on problems never solved by the base model ($+0.52$), whereas performance is lower than GRPO on the easiest $12$--$16/16$ band ($-0.63$). This pattern is consistent with the intended behavior of difficulty gating: teacher guidance can be more beneficial on difficult or intermediate groups, while its value may diminish once the student already solves a problem reliably.

Overall, Table~\ref{tab:bands} shows that the effect of teacher guidance varies substantially with problem difficulty. A fixed additive teacher signal does not account for this heterogeneity, whereas difficulty gating provides a simple mechanism for concentrating supervision on harder rollout groups and reducing its influence as the student becomes more successful. The remaining variation across difficulty bands also suggests that incorporating teacher reliability could further improve the routing mechanism.

\begin{figure}[t]
    \centering
    \includegraphics[width=\linewidth]{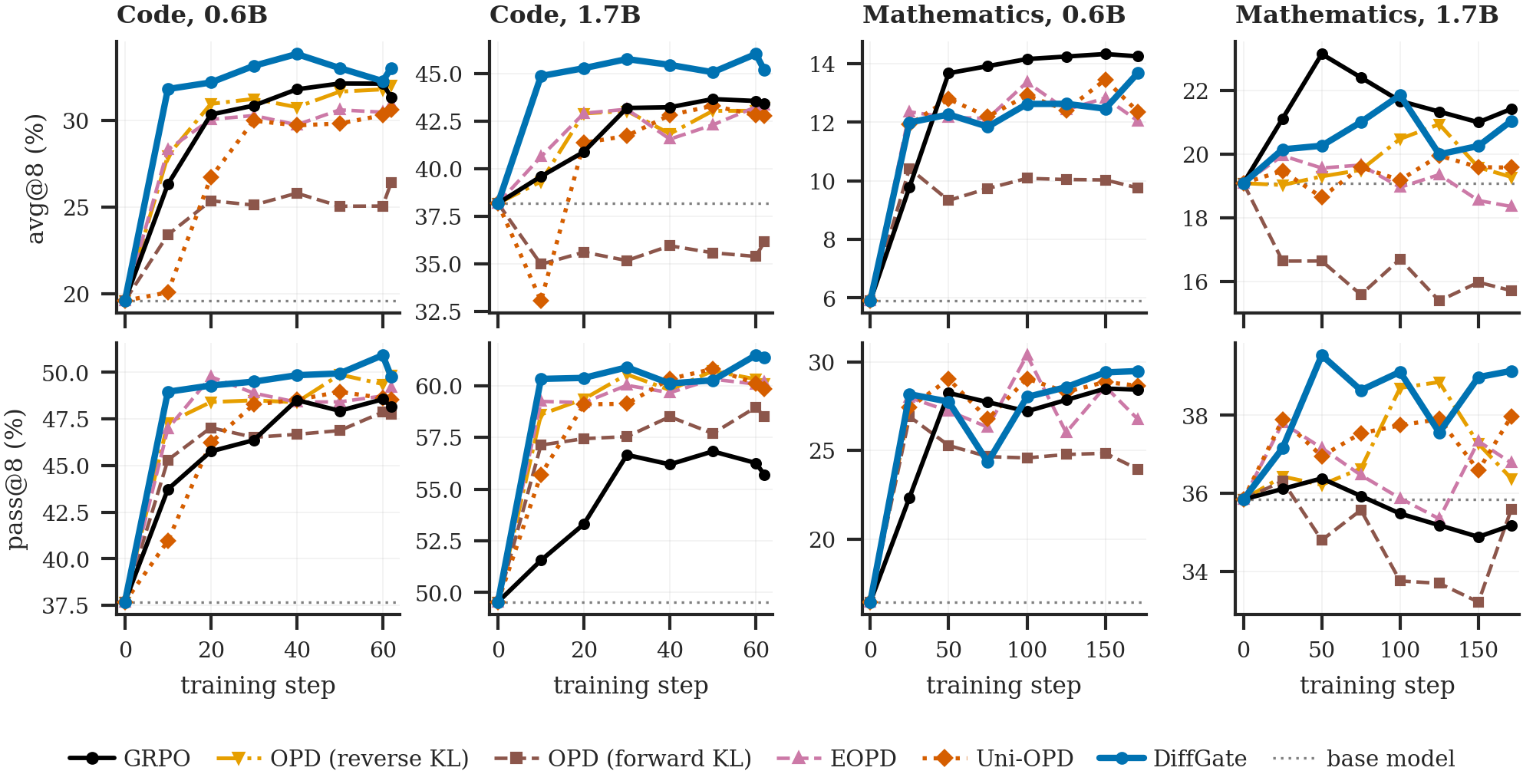}
    \vspace{-0.3cm}
    \caption{\small \textbf{Performance over training.} avg@8 (top) and pass@8 (bottom) on the benchmark sets of Tables~\ref{tab:math} and~\ref{tab:code}; step~$0$ is the base model (dotted line). GRPO and \method\ were evaluated at steps $50$ and $62$ on code, and GRPO at steps $50$, $100$, $150$ and $171$ on mathematics; their other points are filled in as described in the text.}
    \label{fig:stepwise}
    \vspace{-0.2cm}
\end{figure}

\subsection{Training Dynamics and Performance}\label{app:stepwise}
Figure~\ref{fig:stepwise} tracks avg@8 and pass@8 throughout training on the benchmark sets from Tables~\ref{tab:math} and~\ref{tab:code}, with step~$0$ corresponding to the base model. Because only a few checkpoints were saved for the main GRPO and \method\ runs, the curves are supplemented with matched retrains evaluated more frequently. On code, \method\ leads GRPO in pass@8 from the earliest checkpoint and maintains this advantage throughout training, while achieving most of its avg@8 improvement within the first few updates. At 1.7B, GRPO narrows the avg@8 gap later in training, but \method\ retains stronger solution coverage. On mathematics, GRPO reaches its highest avg@8 early and later declines, whereas \method\ maintains higher pass@8 throughout the observed checkpoints. These trends support the main results by showing that \method's gains in solution coverage emerge early and remain stable during training, rather than appearing only at the final checkpoint.

\section{Qualitative Examples}
\label{app:examples}
This section presents illustrative generations underlying the quantitative results in Tables~\ref{tab:math} and~\ref{tab:code}. Every comparison uses the same Qwen3-1.7B student architecture, the same problem, and the same $n=16$ evaluation pool; the two arms therefore differ only in their training objective. Counts denote the number of correct generations out of the same $16$ samples.
We show the available reasoning and final-answer portion of the first sampled response from each arm. The examples were selected from problems exhibiting a large performance gap between GRPO and \method. We do not infer errors that are not visible in the displayed generations.

\subsection{Code Examples}
Code was evaluated by executing the generated program against the reference unit tests. Hence, unlike free-form mathematics evaluation, a large
$16$-versus-$0$ split cannot be explained purely by answer extraction. Both examples below are
ordinary local implementation errors: an off-by-one in a sieve bound and an index-parity confusion. A second failure mode we observed, but do not illustrate here, is GRPO continuing past the requested function to emit its own test code; because the whole program is executed, such extra code can invalidate an otherwise correct implementation.

\begin{exampleheader}{
Example C.1: MBPP+ \hfill
\textnormal{\method: \textbf{16/16} \quad GRPO: \textbf{0/16}}
}
\textbf{Observation.} The task asks for the number of primes \emph{strictly below} $n$. GRPO constructs the sieve through index $n$, so a prime value of $n$ is incorrectly included. Its appended assertion
\texttt{count\_Primes\_nums(5) == 2} also contradicts the implementation it just generated.
\end{exampleheader}

\begin{problembox}
\textbf{Prompt.}
Write a Python function that takes a non-negative number and returns the number of prime numbers less than the given non-negative number.
\end{problembox}

\begin{grpotext}
GRPO initializes a sieve of length $n+1$, marks composites up to $n$, and finally sums the entire Boolean array.
\end{grpotext}

\begin{grpocode}
def count_Primes_nums(n):
    if n < 2:
        return 0

    # Initialize a list to mark non-prime numbers
    is_prime = [True] * (n + 1)
    is_prime[0] = is_prime[1] = False

    for i in range(2, int(n**0.5) + 1):
        if is_prime[i]:
            for j in range(i*i, n+1, i):
                is_prime[j] = False

    return sum(is_prime)

# Test cases
assert count_Primes_nums(5) == 2
assert count_Primes_nums(10) == 4
assert count_Primes_nums(100) == 25
\end{grpocode}

\begin{methodtext}
\method instead constructs a Boolean array of length $n$, so that its valid indices are exactly $0,\ldots,n-1$. It marks composites only within this range and therefore counts primes strictly below $n$.
\end{methodtext}

\begin{methodcode}
def count_Primes_nums(n):
    if n < 2:
        return 0

    # Initialize a boolean list to mark primes
    is_prime = [True] * n
    is_prime[0], is_prime[1] = False, False

    for i in range(2, int(n**0.5) + 1):
        if is_prime[i]:
            for j in range(i*i, n, i):
                is_prime[j] = False

    return sum(is_prime)
\end{methodcode}

\begin{methodtext}
The response additionally notes that the implementation handles the $n<2$ edge case and uses the Sieve of Eratosthenes for efficiency.
\end{methodtext}

\begin{exampleheader}{
Example C.2: HumanEval+ \hfill
\textnormal{\method: \textbf{16/16} \quad GRPO: \textbf{0/16}}
}
\textbf{Observation.}
This example is a local logic error. GRPO selects odd values at \emph{odd} indices, whereas the specification asks for odd values at
\emph{even} indices.
\end{exampleheader}

\begin{problembox}
\textbf{Prompt.}
\begin{quote}
Given a non-empty list of integers, return the sum of all of the odd elements that are in even positions.

\[
\begin{aligned}
\texttt{solution([5,8,7,1])} &\Rightarrow 12,\\
\texttt{solution([3,3,3,3,3])} &\Rightarrow 9,\\
\texttt{solution([30,13,24,321])} &\Rightarrow 0.
\end{aligned}
\]
\end{quote}
\end{problembox}

\begin{grpotext}
GRPO iterates over index--value pairs, but checks
\[
    i \bmod 2 = 1,
\]
thereby selecting odd indices.
\end{grpotext}

\begin{grpocode}
def solution(lst):
    total = 0
    for i, num in enumerate(lst):
        if i % 2 == 1 and num % 2 != 0:
            total += num
    return total
\end{grpocode}

\begin{methodtext}
\method explicitly states the two required conditions:

\begin{enumerate}
    \item the element must be odd, and
    \item its index must be even.
\end{enumerate}

It then iterates through the list, applies both conditions, accumulates the matching values, and returns the resulting sum.
\end{methodtext}

\begin{methodcode}
def solution(lst):
    total = 0
    for i, num in enumerate(lst):
        if num % 2 != 0 and i % 2 == 0:
            total += num
    return total
\end{methodcode}

\subsection{Mathematics Examples}
For both mathematics pairs we additionally require that the reference-answer string does \emph{not} appear anywhere in the losing GRPO response, so each one illustrates a substantive reasoning or final-decision failure rather than a case where the evaluator merely fails to extract an already-correct answer.


\begin{exampleheader}{
Example M.1: OlympiadBench \hfill
\textnormal{\method: \textbf{15/16} \quad GRPO: \textbf{0/16}}
}
\textbf{Observation.}
GRPO makes an off-by-one error in the admissible range of bases.
\end{exampleheader}

\begin{problembox}
Let $T=41$. Compute the number of positive integers $b$ such that the number $T$ has exactly two digits when written in base $b$.

\medskip
\textbf{Reference answer:} $35$.
\end{problembox}

\begin{grpotext}
\textbf{GRPO reasoning.}

For $T$ to have exactly two digits in base $b$, the response writes
\[
    b < T < b^2.
\]

Substituting $T=41$ gives
\[
    b < 41 < b^2.
\]

The response therefore concludes
\[
    \sqrt{41} < b < 41.
\]

Since
\[
    \sqrt{41}\approx 6.403,
\]
it lists the possible integer bases as
\[
    b=7,8,9,\ldots,40.
\]

It then counts
\[
    40-7+1=34
\]
values and outputs
\[
    \boxed{34}.
\]
\end{grpotext}

\begin{methodtext}
\textbf{\method\ reasoning.}

\method instead uses
\[
    b \leq 41 < b^2.
\]

The right inequality gives
\[
    b^2>41
    \quad\Longrightarrow\quad
    b>\sqrt{41}
    \approx 6.403.
\]

Because $b$ is a positive integer,
\[
    b\geq 7.
\]

The left inequality gives
\[
    b\leq41.
\]

Hence the complete range is
\[
    7\leq b\leq41.
\]

The number of integers in this range is
\[
    41-7+1=35.
\]

Therefore,
\[
    \boxed{35}.
\]
\end{methodtext}

\begin{exampleheader}{
Example M.2: OlympiadBench \hfill
\textnormal{\method: \textbf{10/16} \quad GRPO: \textbf{0/16}}
}
\textbf{Observation.}
GRPO identifies both valid solutions during its reasoning but outputs only one in the final answer, whereas \method\ retains the complete solution set.
\end{exampleheader}

\begin{problembox}
Find all functions
\[
    f:\mathbb{Q}\rightarrow\mathbb{Q}
\]
such that, for all rational $x$ and $y$,
\[
    f(xf(x)+y)=f(y)+x^2.
\]

\medskip
\textbf{Reference answer:}
\[
    f(x)=x
    \qquad\text{or}\qquad
    f(x)=-x.
\]
\end{problembox}

\begin{grpotext}
\textbf{GRPO reasoning near the conclusion.}

The response correctly identifies the two candidate solutions,
\[
    f(x)=x
    \qquad\text{and}\qquad
    f(x)=-x,
\]
and recognizes that both satisfy the functional equation.

However, when producing its final answer, it retains only
\[
    \boxed{f(x)=x}.
\]

Thus, although the reasoning contains both valid solutions, the explicit final answer is incomplete.
\end{grpotext}

\begin{methodtext}
\textbf{\method\ reasoning near the conclusion.}

The response identifies both functions satisfying the equation,
\[
    f(x)=x
    \qquad\text{and}\qquad
    f(x)=-x,
\]
and preserves the complete solution set in its final answer:
\[
    \boxed{
        f(x)=x
        \quad\text{or}\quad
        f(x)=-x
    }.
\]
\end{methodtext}

\section{Limitations and Future Work}
Our experiments focus on Qwen3 students and verifiable mathematics and code tasks, with models up to 1.7B parameters and a limited 4B evaluation. Extending \method\ to other model families, larger scales, and broader reasoning domains remains future work. \method\ emphasizes difficult prompts, where teacher reliability may also decrease, particularly under larger capability gaps. On held-out code problems never solved by the base student, the 8B teacher achieves pass@16 of $42$--$50\%$, compared with $94$--$98\%$ on the remaining difficulty bands (see Appendix Section ~\ref{app:bands}). Incorporating online estimates of teacher reliability into the gate is a promising direction.

Although \method\ requires teacher scoring in addition to grouped rollouts, its overhead over GRPO remains modest (see Appendix Section ~\ref{app:cost}). Moreover, \method\ remains effective with smaller rollout groups, suggesting that teacher guidance remains useful even with reduced grouped generation. Future work could explore compute-efficient settings with fewer rollouts and adaptive teacher usage.

\section{Use of Large Language Models (LLMs)}
Large language models (LLMs) were used solely to assist with manuscript writing, including grammar, clarity, wording, and organization. All LLM-assisted content was reviewed and revised by the authors, who take full responsibility for the final technical claims, interpretation, and presentation.

\end{document}